\documentclass[11pt,reqno,english]{scrartcl}
\usepackage[T1]{fontenc}
\usepackage{verbatim}
\usepackage{amsmath}
\usepackage{amsthm}
\usepackage{amssymb}
\usepackage{undertilde}
\usepackage{graphicx}
\usepackage{xargs}[2008/03/08]

\makeatletter

\newcommand{\noun}[1]{\textsc{#1}}

\numberwithin{figure}{section}
\numberwithin{equation}{section}
\newcommand{\lyxaddress}[1]{
	\par {\raggedright #1
	\vspace{1.4em}
	\noindent\par}
}
\theoremstyle{plain}
\newtheorem{thm}{\protect\theoremname}[section]
\theoremstyle{definition}
\newtheorem{example}[thm]{\protect\examplename}
\theoremstyle{plain}
\newtheorem{prop}[thm]{\protect\propositionname}
\theoremstyle{remark}
\newtheorem{rem}[thm]{\protect\remarkname}
\theoremstyle{plain}
\newtheorem{cor}[thm]{\protect\corollaryname}

\usepackage[all]{xy}
\usepackage{tikz-cd}
\usepackage{libertine}
\usepackage[libertine,bigdelims]{newtxmath}
\usepackage[cal=boondoxo,bb=boondox,frak=boondox]{mathalfa}

\makeatother

\usepackage{babel}
\providecommand{\corollaryname}{Corollary}
\providecommand{\examplename}{Example}
\providecommand{\propositionname}{Proposition}
\providecommand{\remarkname}{Remark}
\providecommand{\theoremname}{Theorem}

\begin{document}

\global\long\def\ga{\alpha}%
\global\long\def\gb{\beta}%
\global\long\def\ggm{\gamma}%
\global\long\def\go{\omega}%
\global\long\def\ge{\epsilon}%
\global\long\def\gs{\sigma}%
\global\long\def\gd{\delta}%
\global\long\def\gD{\Delta}%
\global\long\def\vph{\varphi}%
\global\long\def\gf{\varphi}%
\global\long\def\gk{\kappa}%
\global\long\def\gl{\lambda}%

\global\long\def\eps{\varepsilon}%
\global\long\def\epss#1#2{\varepsilon_{#2}^{#1}}%
\global\long\def\ep#1{\eps_{#1}}%

\global\long\def\wh#1{\widehat{#1}}%

\global\long\def\spec#1{\textsf{#1}}%

\global\long\def\ui{\wh{\boldsymbol{\imath}}}%
\global\long\def\uj{\wh{\boldsymbol{\jmath}}}%
\global\long\def\uk{\widehat{\boldsymbol{k}}}%

\global\long\def\uI{\widehat{\mathbf{I}}}%
\global\long\def\uJ{\widehat{\mathbf{J}}}%
\global\long\def\uK{\widehat{\mathbf{K}}}%

\global\long\def\bs#1{\boldsymbol{#1}}%
\global\long\def\vect#1{\mathbf{#1}}%
\global\long\def\bi#1{\textbf{\emph{#1}}}%

\global\long\def\uv#1{\widehat{\boldsymbol{#1}}}%
\global\long\def\cross{\times}%

\global\long\def\ddt{\frac{\dee}{\dee t}}%
\global\long\def\dbyd#1{\frac{\dee}{\dee#1}}%
\global\long\def\dby#1#2{\frac{\partial#1}{\partial#2}}%

\global\long\def\vct#1{\mathbf{#1}}%

\global\long\def\partialby#1#2{\frac{\partial#1}{\partial x^{#2}}}%
\newcommandx\parder[2][usedefault, addprefix=\global, 1=]{\frac{\partial#2}{\partial#1}}%

\global\long\def\oneto{1,\dots,}%
\global\long\def\mi#1{\boldsymbol{#1}}%
\global\long\def\mii{\mi I}%

\global\long\def\fall{,\quad\text{for all}\quad}%

\global\long\def\reals{\mathbb{R}}%

\global\long\def\rthree{\reals^{3}}%
\global\long\def\rsix{\reals^{6}}%
\global\long\def\rn{\reals^{n}}%
\global\long\def\rt#1{\reals^{#1}}%

\global\long\def\les{\leqslant}%
\global\long\def\ges{\geqslant}%

\global\long\def\dee{\textrm{d}}%
\global\long\def\di{d}%

\global\long\def\from{\colon}%
\global\long\def\tto{\longrightarrow}%

\global\long\def\abs#1{\left|#1\right|}%

\global\long\def\isom{\cong}%

\global\long\def\comp{\circ}%

\global\long\def\cl#1{\overline{#1}}%

\global\long\def\fun{\varphi}%

\global\long\def\interior{\textrm{Int}\,}%

\global\long\def\sign{\textrm{sign}\,}%
\global\long\def\sgn#1{(-1)^{#1}}%
\global\long\def\sgnp#1{(-1)^{\abs{#1}}}%

\global\long\def\dimension{\textrm{dim}\,}%

\global\long\def\esssup{\textrm{ess}\,\sup}%

\global\long\def\ess{\textrm{{ess}}}%

\global\long\def\kernel{\mathop{\textrm{Kernel}}}%

\global\long\def\support{\textrm{supp}\,}%

\global\long\def\image{\mathrm{Image}\,}%

\global\long\def\diver{\mathop{\textrm{div}}}%

\global\long\def\sp{\mathop{\textrm{span}}}%

\global\long\def\resto#1{|_{#1}}%
\global\long\def\incl{\iota}%
\global\long\def\iden{\imath}%
\global\long\def\idnt{\textrm{Id}}%
\global\long\def\rest{\rho}%
\global\long\def\extnd{e_{0}}%

\global\long\def\proj{\textrm{pr}}%

\global\long\def\ino#1{\int_{#1}}%

\global\long\def\half{\frac{1}{2}}%
\global\long\def\shalf{{\scriptstyle \half}}%
\global\long\def\third{\frac{1}{3}}%

\global\long\def\empt{\varnothing}%

\global\long\def\paren#1{\left(#1\right)}%
\global\long\def\bigp#1{\bigl(#1\bigr)}%
\global\long\def\biggp#1{\biggl(#1\biggr)}%
\global\long\def\Bigp#1{\Bigl(#1\Bigr)}%

\global\long\def\braces#1{\left\{  #1\right\}  }%
\global\long\def\sqbr#1{\left[#1\right]}%
\global\long\def\anglep#1{\left\langle #1\right\rangle }%

\global\long\def\lsum{{\textstyle \sum}}%

\global\long\def\bigabs#1{\bigl|#1\bigr|}%

\global\long\def\lisub#1#2#3{#1_{1}#2\dots#2#1_{#3}}%

\global\long\def\lisup#1#2#3{#1^{1}#2\dots#2#1^{#3}}%

\global\long\def\lisubb#1#2#3#4{#1_{#2}#3\dots#3#1_{#4}}%

\global\long\def\lisubbc#1#2#3#4{#1_{#2}#3\cdots#3#1_{#4}}%

\global\long\def\lisubbwout#1#2#3#4#5{#1_{#2}#3\dots#3\widehat{#1}_{#5}#3\dots#3#1_{#4}}%

\global\long\def\lisubc#1#2#3{#1_{1}#2\cdots#2#1_{#3}}%

\global\long\def\lisupc#1#2#3{#1^{1}#2\cdots#2#1^{#3}}%

\global\long\def\lisupp#1#2#3#4{#1^{#2}#3\dots#3#1^{#4}}%

\global\long\def\lisuppc#1#2#3#4{#1^{#2}#3\cdots#3#1^{#4}}%

\global\long\def\lisuppwout#1#2#3#4#5#6{#1^{#2}#3#4#3\wh{#1^{#6}}#3#4#3#1^{#5}}%

\global\long\def\lisubbwout#1#2#3#4#5#6{#1_{#2}#3#4#3\wh{#1}_{#6}#3#4#3#1_{#5}}%

\global\long\def\lisubwout#1#2#3#4{#1_{1}#2\dots#2\widehat{#1}_{#4}#2\dots#2#1_{#3}}%

\global\long\def\lisupwout#1#2#3#4{#1^{1}#2\dots#2\widehat{#1^{#4}}#2\dots#2#1^{#3}}%

\global\long\def\lisubwoutc#1#2#3#4{#1_{1}#2\cdots#2\widehat{#1}_{#4}#2\cdots#2#1_{#3}}%

\global\long\def\twp#1#2#3{\dee#1^{#2}\wedge\dee#1^{#3}}%

\global\long\def\thp#1#2#3#4{\dee#1^{#2}\wedge\dee#1^{#3}\wedge\dee#1^{#4}}%

\global\long\def\fop#1#2#3#4#5{\dee#1^{#2}\wedge\dee#1^{#3}\wedge\dee#1^{#4}\wedge\dee#1^{#5}}%

\global\long\def\idots#1{#1\dots#1}%
\global\long\def\icdots#1{#1\cdots#1}%

\global\long\def\norm#1{\|#1\|}%

\global\long\def\nonh{\heartsuit}%

\global\long\def\nhn#1{\norm{#1}^{\nonh}}%

\global\long\def\trps{^{{\scriptscriptstyle \textsf{T}}}}%

\global\long\def\testfuns{\mathcal{D}}%

\global\long\def\ntil#1{\tilde{#1}{}}%

\global\long\def\alt{\mathfrak{A}}%

\global\long\def\pou{\eta}%

\global\long\def\ext{{\textstyle \bigwedge}}%
\global\long\def\forms{\Omega}%

\global\long\def\dotwedge{\dot{\mbox{\ensuremath{\wedge}}}}%

\global\long\def\vel{\theta}%

\global\long\def\contr{\raisebox{0.4pt}{\mbox{\ensuremath{\lrcorner}}}\,}%

\global\long\def\lie{\mathcal{L}}%

\global\long\def\L#1{L\bigl(#1\bigr)}%

\global\long\def\vvforms{\ext^{\dims}\bigp{T\spc,\vbts^{*}}}%

\global\long\def\spc{\mathcal{S}}%
\global\long\def\sptm{\mathcal{E}}%
\global\long\def\evnt{e}%
\global\long\def\frame{\Phi}%

\global\long\def\timeman{\mathcal{T}}%
\global\long\def\zman{t}%
\global\long\def\dims{n}%
\global\long\def\m{\dims-1}%
\global\long\def\dimw{m}%

\global\long\def\wc{z}%

\global\long\def\util#1{\raisebox{-5pt}{\ensuremath{{\scriptscriptstyle \sim}}}\!\!\!#1}%

\global\long\def\utilJ{\util J}%

\global\long\def\utilRho{\util{\rho}}%

\global\long\def\body{B}%
\global\long\def\man{\mathcal{M}}%
\global\long\def\var{\mathcal{V}}%

\global\long\def\bdry{\partial}%

\global\long\def\gO{\varOmega}%

\global\long\def\reg{\mathcal{R}}%
\global\long\def\bdrr{\bdry\reg}%

\global\long\def\bdom{\bdry\gO}%

\global\long\def\bndo{\partial\gO}%

\global\long\def\pis{x}%
\global\long\def\xo{\pis_{0}}%

\global\long\def\pib{X}%

\global\long\def\pbndo{\Gamma}%
\global\long\def\bndoo{\pbndo_{0}}%
 
\global\long\def\bndot{\pbndo_{t}}%

\global\long\def\cloo{\cl{\gO}}%

\global\long\def\nor{\mathbf{n}}%

\global\long\def\dA{\,\dee A}%

\global\long\def\dV{\,\dee V}%

\global\long\def\eps{\varepsilon}%

\global\long\def\affsp{\mathbf{A}}%
\global\long\def\pt{p}%

\global\long\def\vbase{e}%
\global\long\def\sbase{\mathbf{e}}%
\global\long\def\msbase{\mathfrak{e}}%
\global\long\def\vect{v}%

\global\long\def\vf{w}%

\global\long\def\avf{u}%

\global\long\def\stn{\varepsilon}%

\global\long\def\rig{r}%

\global\long\def\rigs{\mathcal{R}}%

\global\long\def\qrigs{\!/\!\rigs}%

\global\long\def\qd{\!/\,\!\kernel\diffop}%

\global\long\def\dis{\chi}%

\global\long\def\csa{\mathcal{Q}}%
\global\long\def\csb{\mathcal{P}}%
\global\long\def\fk{\vph}%
\global\long\def\cs{p}%
\global\long\def\conf{\kappa}%

\global\long\def\jac{J}%
\global\long\def\vs{\mathcal{V}}%
\global\long\def\avs{\mathcal{U}}%
\global\long\def\svs{\mathcal{W}}%
\global\long\def\v{v}%
\global\long\def\av{u}%

\global\long\def\fc{f}%
\global\long\def\afc{g}%
\global\long\def\du{^{*}}%

\global\long\def\opt{\gO}%

\global\long\def\yi{\mathbf{M}}%
\global\long\def\fs{S}%
\global\long\def\h{_{\mathrm{H}}}%

\global\long\def\g{\mathbf{g}}%

\title{Optimization of Robot Grasping Forces and Worst Case Loading}
\author{Or Elmackias, Tami Zaretzky, and Reuven Segev}
\maketitle

\lyxaddress{Department of Mechanical Engineering\\
Ben-Gurion University of the Negev\\
Beer-Sheva, Israel\\
Corresponding author: rsegev@bgu.ac.il}


\medskip{}

\emph{\noindent Mathematics Subject Classification (MSC):} 70E60;
74P10; 46N10
\begin{abstract}
\emph{Abstract: }We consider the optimization of the vector of grasping
forces that support a known generalized force acting on the grasped
object---a rigid body or a mechanism. Working in the framework of
finite dimensional normed vector spaces and their dual spaces, the
cost function to be minimized is assumed to be a norm on the space
of grasping forces. We present an expression for the optimum which
depends on the external force and the kinematics of the grasping system.
Next, assuming that optimal grasping forces are applied using force
control, and assuming that there is a bound on the norm of the admissible
grasping forces, we characterize the largest norm of an external force
that the grasping system may support, that is, the norm of the worst
case loading that may be applied and still be supported. A few simple
examples are given for the sake of illustration.

\end{abstract}

\section{Introduction}

The mechanical analysis of robot grasping is under ongoing active
research. See, for example, the comprehensive review of the subject
in \cite{Burdick2019}, published recently, and references cited
therein. Within this general area of research, various optimization
problems may be considered. In \cite[Chatper 9]{Burdick2019}, the
minimal number of fingers in the grasping mechanism is studied. In
\cite[Chatper 12]{Burdick2019}, the problem of ``gentlest grasp''
is defined, where the maximum grasping force is minimized. Another
study which is relevant to the present work is \cite{Bologni1988},
where the author considers a planar rigid body and searches for the
optimal 3 points of grasping normal and friction forces so that the
maximal external force may be supported. Additional work on problems
related to optimization of grasping may found in \cite{Ohka2010,Fasoulas2012}.

This article is concerned with the optimization of grasping forces
that robots apply to objects they carry, where the objects may be
rigid bodies or mechanisms. Assume that the external forces on an
object are given, and that there are more grasping force components
than the number of degrees of freedom of the object. In such cases,
the grasping forces cannot be determined by the equations of mechanics
alone---there will be an infinite number of solutions to the equations.
Thus, it might be desirable, to find among all grasping forces that
satisfy the equations of mechanics, those forces that satisfy some
optimality conditions.

For example, a fragile object may be held at a number of points, so
that one may wish to minimize the largest applied grasping force---gentlest
grasp. In the general case, considered in Section \ref{sec:Optimization},
the optimality condition is given in terms of a norm on the space
grasping forces. The main result, given in Proposition \ref{prop:Optimal},
describes the optimum in terms of the applied external force and the
linear mapping, $\jac$, that gives the generalized velocities at
the points where the object is held in terms of the generalized velocities
of the object.

Once the result on optimal reactions is at hand, one may consider
the following problem. Conceivably, one could use a force controlled
grasping system in order to realize optimal grasping of an object
for any given external load. However, due to limitations of the force
control system and in order to prevent damage to the object, which
may be caused by grasping, the norm of admissible grasping forces
may be limited to some value $\yi$. The question then arises as to
the nature of the set of admissible external forces on the object.
That is, how can we characterize the external forces that may be supported
by grasping forces, the norms of which do not exceed $\yi$? A related
problem is concerned with worst case loading. How can one characterize
forces that pose the greatest risk to grasping?

These issues are considered in Section \ref{sec:Load_Capacity} and
the main result is given in Proposition \ref{prop:LoadCap}. Simply
put, there is a number $R>0$, the \textit{grasping system sensitivity,
}such that the admissible forces, $\fc$, are those satisfying the
condition 
\begin{equation}
\norm{\fc}\les\frac{\yi}{R}.
\end{equation}
The system's sensitivity, $R$, is determined by the linear mapping
$\jac$, mentioned above. In particular, $R$ depends only on the
geometry of the grasping.

The methods we use in this paper are analogous to those used in previous
work \cite{LoadCap07,Falach2009} in the context of continuum mechanics.
Section \ref{sec:Notation-and-Prliminaries} below overviews the framework
for the mechanics of grasping within which we work. Of special importance
is the natural assumption that the mapping $\jac$, relating the generalized
velocities of the object to the generalized velocities of the supported
points, is injective. Section \ref{sec:Optimization} presents the
optimization problem in terms of a norm on the space of grasping forces.
The main result of this section, Proposition \ref{prop:Optimal} is
based on the finite dimensional version of the Hahn-Banach theorem.
Bounded grasping forces and worst case loadings are studied in Section
\ref{sec:Load_Capacity}. The notion of sensitivity of the grasping
to a certain external force is defined, and is related to the \textit{factor
of safety} of the grasping. These objects are intimately related to
the Minkowski functional for a certain convex set, $K$, associated
with the admissible grasping forces. The main result, as described
above, is presented in Proposition \ref{prop:LoadCap}. Since one
of the methods of computing the optimal reactions is based on the
construction of the proof of the finite dimensional Helly-Hahn-Banach
theorem (Helly \cite{Helly1912} was the first to prove the finite-dimensional
version), and in order to make the manuscript self-contained, we overview
this construction in Appendix \ref{sec:H-B}.

In order to illustrate the approach and results presented in text,
we consider in Section \ref{sec:Optimization} and \ref{sec:Load_Capacity}
a number of very simple examples.

\section{The mechanical framework for the analysis\label{sec:Notation-and-Prliminaries}}

This section will present the mechanical and mathematical settings
for the grasping problem that we wish to optimize. It relies on the
geometric point of view of the mechanical system involved.

\subsection{Objects, configurations and constraints}

Our analysis is focused on an \emph{object}---a mechanical system
that is grasped by the robot. In most applications, this mechanical
system will be modeled as a rigid body. However, the following analysis
holds also when the grasped object is a mechanism. The kinematics
of the object is modeled by an $n$-dimensional differentiable manifold,
the \emph{configuration space}, $\csa$, where $n$ is the number
of degrees of freedom of the object. A typical configuration of the
object will be denoted by $\conf\in\csa$. The physical space will
be represented by $\rthree$.

The object is supported by constraints that determine its configurations.
As a typical example, the locations in space of a fixed number of
points belonging to the objects may be constrained by the robot grippers.
(See illustration in Figure \ref{fig:Constraints}.)

\begin{figure}
\begin{centering}
\includegraphics[scale=0.7]{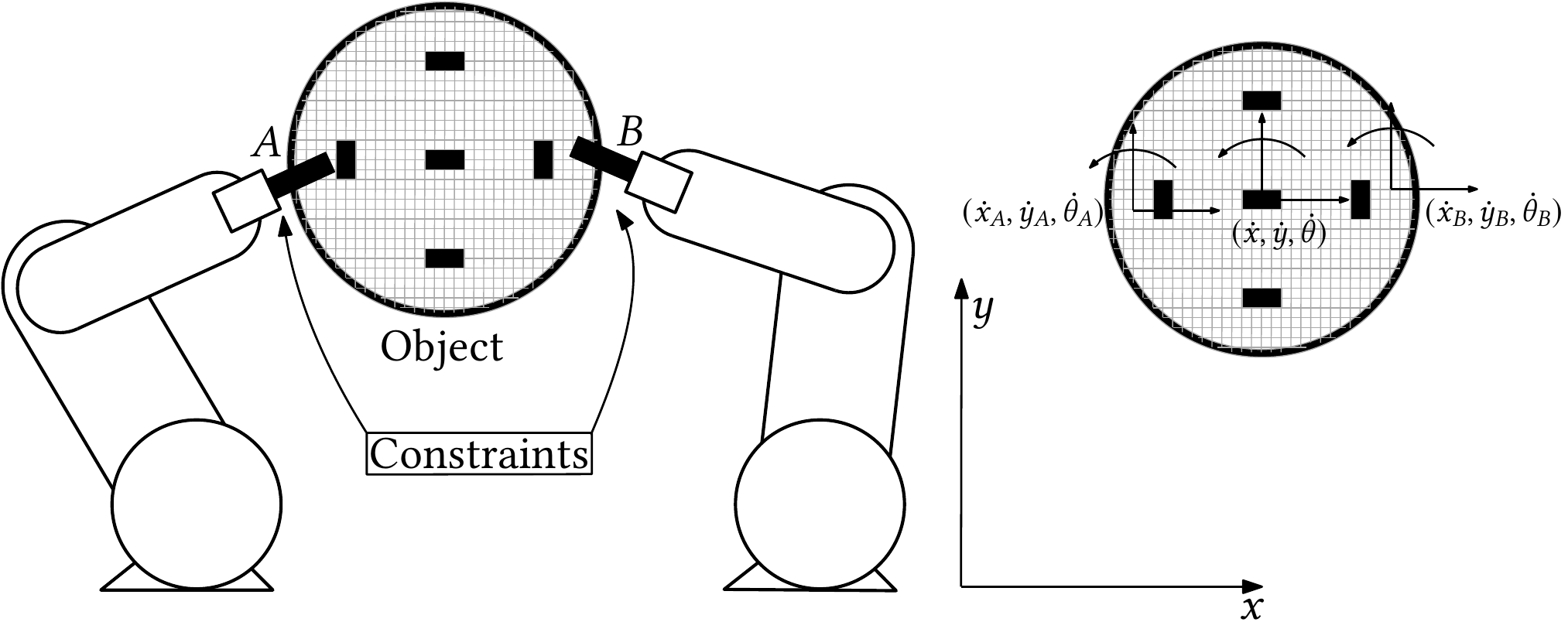}
\par\end{centering}
\caption{\label{fig:Constraints}Illustrating an object and constraints. In
this simple example, $\protect\csa=\protect\reals^{2}\times S^{1}=\{(x,y,\theta)\}$,
$\protect\csb=(\protect\reals^{2}\times S^{1})^{2}=\{(x_{A},y_{A},\theta_{A},x_{B},y_{B},\theta_{B})\}$,
$\protect\vs=\protect\reals^{3}=\{(\dot{x},\dot{y},\dot{\theta})\}$
and $\protect\avs=\protect\rsix=\{(\dot{x}_{A},\dot{y}_{A},\dot{\theta}_{A},\dot{x}_{B},\dot{y}_{B},\dot{\theta}_{B})\}$.}
\end{figure}
It is natural, therefore, to assume that a given configuration $\conf\in\csa$
determines the corresponding state of the constraints. For example,
if the object is a rigid body which is constrained by the locations
of $m_{c}$ points in it, any configuration of the object will determine
$m_{c}$ points in space, an element of $\reals^{3m_{c}}$. We will
denote a typical state of the constraints by $\cs$ and will denote
the collection of all states of the constraints---the \emph{constraint
space---}by $\csb$. Thus, for the example, $\csb=\reals^{3m_{c}}$.
In the general case, the constraint space $\csb$ is assumed to be
a differentiable manifold of dimension $m$. Since fixing the state
of the constraints should determine a unique state of the object,
it is assumed that $\dimension\csb=m\ges\dimension\csa=n$, so that
the constraints can eliminate all degrees of freedom the object might
have. Thus, the constraint state determined by a configuration of
the object is given through a mapping---the\emph{\noun{ }}\emph{constraint
mapping,}
\begin{equation}
\vph:\csa\tto\csb.
\end{equation}
It is assumed that the mapping $\fk$ is differentiable. Moreover,
the assumption that a state of the constraints determines a unique
configuration of the object, implies that the mapping $\vph$ is injective.
It is observed that the assumption that $\vph$ is injective, is different
from what one expects from the forward kinematics mapping in robotics,
where it is accepted that the inverse kinematics problem does not
have a unique solution.

We note that the dimension of the constraint space, $m$ may be strictly
greater than the dimension $n$ of the configurations space. In such
a case, the object is over-constrained by the grasping. If $\cs\in\image\fk$,
that is, there is some configuration $\conf\in\csa$ such that $\cs=\fk(\conf)$,
we will say that $\cs$ is a \emph{compatible constraint state} and
that $\cs$ \emph{is compatible with} $\conf$. For the case where
$m>n$, strictly, it is clear that not every constraint state is compatible.
\begin{example}
Consider the two-dimensional system considered in Figure \ref{fig:Constraints}.
Here, $\csa=\reals^{2}\times S^{1}=\{(x,y,\theta)\}$ and $\csb=(\reals^{2}\times S^{1})^{2}=\{(x_{A},y_{A},\theta_{A},x_{B},y_{B},\theta_{B})\}$.
It is noted that for a generic element of $\csb$, there need not
be any relation between $(x_{A},y_{A},\theta_{A})$ and $(x_{B},y_{B},\theta_{B})$.
Evidently, kinematics of rigid bodies provide the relation between
the components of compatible elements of $\csb$.
\end{example}

\subsection{Velocities, external forces and supporting forces}

Modeling the configuration space of the object as a differentiable
manifold offers a natural representation of \emph{generalized velocities}
at any configuration $\conf\in\csa$, as elements of the tangent vector
space $T_{\conf}\csa$ to the configuration manifold at the point
$\conf$. Similarly, elements of $T_{\cs}\csb$ are interpreted as
\emph{constraint }or \emph{support velocities}.

The differentiability assumption for the constraint mapping implies
that for any $\conf\in\csa$, we have a linear mapping---the tangent
mapping
\begin{equation}
T_{\conf}\vph:T_{\conf}\csa\tto T_{\fk(\conf)}\csb
\end{equation}
that assigns a constraint velocity to each generalized velocity of
the object. The mapping $T_{\conf}\fk$ is represented in coordinates
by the Jacobian of the mapping representing $\fk$ (as we describe
below). Henceforth, we strengthen the assumption that $\fk$ is injective
and assume further that the linear $T_{\conf}\fk$ is injective. The
property that $T_{\conf}\fk$ is injective will be referred to as
\emph{kinematic determinacy.}

In what follows, we will be concerned with the grasping problem for
one particular configuration of the object. Hence, to simplify the
notation we will use $\vs$ for $T_{\conf}\csa$, $\avs$ for $T_{\vph(\conf)}\csb$,
and $\jac$ for $T_{\conf}\fk$. Thus, the foregoing linear injective
mapping may be written as 
\begin{equation}
\jac:\vs\tto\avs.
\end{equation}

In analogy with the terminology for configurations, we say that a
generalized velocity of the constraints $\av$ is \emph{compatible
}if $\av\in\image\jac$. We say that $\av$ \emph{is compatible with
}the generalized velocity $\v\in\vs$ of the object if $\av=\jac(\v)$.

In general, we adopt here the geometric point of view of statics whereby
for a mechanical system having a configuration space $\mathcal{M}$,
a \emph{generalized force $\fc$ at the configuration $q\in\mathcal{M}$}
is defined as a linear functional---a \emph{covector}
\begin{equation}
\fc:T_{q}\mathcal{M}\tto\reals.
\end{equation}
Thus, generalized forces act linearly on generalized velocities to
produce real numbers. The action $f(v)$ of the force $\fc$ on the
generalized velocity $\v$, is interpreted as the power produced by
the force for the velocity $\v$. Thus, forces at the configuration
$q$ belong to the dual space
\begin{equation}
(T_{q}\mathcal{M})\du:=\{\fc:T_{q}\mathcal{M}\tto\reals\mid f\text{ is linear}\}.
\end{equation}

This applies, in particular, to forces acting on the object. Thus,
a generalized force on the object at a configuration $\conf\in\csa$
is an element $f\in(T_{\conf}\csa)\du=\vs\du$. Similarly, a generalized
force $\afc$ acting on the constraints (for example, due to forces
exerted by the object) at a constrain state $\cs\in\csb$ is an element
of the dual space $(T_{\cs}\csb)\du$. In case $\cs=\fk(\conf)$,
$(T_{\cs}\csb)\du=\avs\du$ and $\afc\in\avs\du$.

Given a force $\afc\in\avs\du$, it is natural to interpret its opposite,
$-\afc$, as the force exerted by the constraints on the object.

The mapping $\jac=T_{\conf}\csa$ induces the dual (adjoint) mapping
\begin{equation}
\jac\du:\avs\du\tto\vs\du,
\end{equation}
 whereby a generalized force $\afc$ on the supports determines an
external force $\fc=\jac\du(\afc)$ on the object, which is defined
by
\begin{equation}
\fc(\v)=\jac\du(\afc)(\v)=\afc(\jac(\v)),\qquad\text{or equivalently,}\qquad\fc=\jac\du(\afc)=\afc\comp\jac.
\end{equation}

Let $\fc\in\vs\du$ be a generalized force acting on the object and
let $\afc\in\avs\du$ be a generalized force acting on the supports.
We will say that $\afc$ is \emph{consistent with }$\fc$ if 
\begin{equation}
\fc=\jac\du(\afc).\label{eq:Eq_Eq}
\end{equation}
Thus, $\afc$ is consistent with $\fc$ if the virtual powers that
they produce are equal for pairs of compatible virtual velocities---the
principle of virtual work. This is equivalent to the statement that
$\fc$ is \emph{in equilibrium} with $-\afc$ and we will refer to
Equation (\ref{eq:Eq_Eq}) as the \emph{consistency condition} or
the \emph{equilibrium equation}.

It is noted that the foregoing geometric framework applies to the
supports and reactions of any fully constrained mechanism.

\subsection{Representation by components}

Let $\lisup{\conf},n$ be a coordinate system in a neighborhood of
$\conf_{0}$ in $\csa$ and let $\lisup{\cs},m$ be a coordinate system
in a neighborhood of $p_{0}=\fk(\conf_{0})$ in $\csb$. In particular,
$\conf_{0}$ is represented by the coordinates $\lisup{\conf_{0}},n$
and $\cs_{0}$ is represented by the coordinates $\lisup{\cs_{0}},m$.
Then, in a neighborhood of $\conf_{0}$, the constraint mapping may
be represented in the form
\begin{equation}
\cs^{i}=\cs^{i}(\lisup{\conf},n)=\cs^{i}(\conf^{\ga})=\fk^{i}(\conf^{\ga}),\qquad\ga=\oneto n,\,\,i=\oneto m.
\end{equation}
(Note that in the last equation we mix the notation for the real valued
functions $\fk^{i}$with the notation for the variables $\cs^{i}$.)
In particular, $\cs_{0}^{i}=\cs^{i}(\conf_{0}^{\ga})=\fk^{i}(\conf_{0}^{\ga})$.

The coordinate neighborhoods described above make it possible to specify
tangent vectors and virtual velocities by components. Thus, a generalized
velocity $\v\in\vs$ of the object will be represented by $(\lisup{\v},n)\in\reals^{n}$,
and a generalized velocity of the constraints, $\av\in\avs$, will
be represented by $(\lisup{\av},m)\in\reals^{m}$.

Using this coordinate representation, the tangent mapping, $\jac=T_{\conf_{0}}\fk$,
of the constraint mapping at $\conf_{0}$ is represented by Jacobian
matrix
\begin{equation}
\jac_{\ga}^{i}:=\parder[\conf^{\ga}]{\cs^{i}}(\conf_{0}^{\gb})=\parder[\conf^{\ga}]{\fk^{i}}(\conf_{0}^{\gb}),\qquad\ga,\gb=\oneto n,\,\,i=\oneto m.
\end{equation}
It follows that the condition that $\v\in\vs$ and $\av\in\avs$ are
compatible may be expressed in the form
\begin{equation}
\av^{i}=\jac_{\ga}^{i}\v^{\ga},\qquad i=\oneto n,\label{eq:vel_comp_comp}
\end{equation}
where the summation convention is used for the repeated index $\ga=\oneto n$.

As an element in the dual space, a force $\fc$ is represented by
$(\lisub{\fc},n)\in\reals^{n}$ and its action on a generalized velocity
$\v\in\vs$ is given by
\begin{equation}
\fc(\v)=\fc_{\ga}\v^{\ga}.
\end{equation}
Similarly, $\afc\in\avs\du$ is represented by an element $(\lisub{\afc},m)\in\reals^{m}$
and its action on a generalized velocity of the constraints $\av$
is given by 
\begin{equation}
\afc(\av)=\afc_{i}\av^{i}.
\end{equation}

Finally the consistency condition for a force $\fc\in\vs\du$ and
a generalized force on the constraints $\afc\in\avs\du$ is
\begin{equation}
\fc_{\ga}\v^{\ga}=\afc_{i}\jac_{\ga}^{i}\v^{\ga}
\end{equation}
for every vector $\v\in\vs$. As expected, the resulting consistency
(or equilibrium) condition is
\begin{equation}
\fc_{\ga}=\afc_{i}\jac_{\ga}^{i},\qquad\ga=\oneto n.\label{eq:Equil_comp}
\end{equation}

\section{\label{sec:Optimization}Static indeterminacy and optimization}

As mentioned above, $m=\dimension\avs\ges n=\dimension\vs$. Since
the dimension of a dual space is equal to the dimension of the primal
space $\dimension\avs\du\ges\dimension\vs\du$. This implies that
the consistency condition, as represented in Equation (\ref{eq:Equil_comp}),
may be viewed, for a given $\fc\in\vs\du$, as a system of linear
$n$ equations with $m\ges n$ unknowns. Thus, for the case where
$m>n$, the force $\fc$ cannot determine a unique constraint force
$\afc\in\avs\du$. This is the origin of what is usually referred
to as \emph{static indeterminacy}. It is noted that static indeterminacy
is a result of the kinematic determinacy with $m>n$, strictly.

The foregoing observation implies that for a given force $\fc\in\vs\du$
on the object, there is a collection 
\begin{equation}
\jac^{*-1}\{\fc\}=\{\afc\in\avs\du\mid\jac\du(\afc)=\fc\}
\end{equation}
of forces on the constraints with which $\fc$ is in equilibrium.
It is meaningful, therefore, to look for an element in the collection
$\jac^{*-1}\{\fc\}$, which is optimal in some sense.

In this paper, the optimization criterion, the value which one wishes
to minimize, is assumed to be represented by a norm on the space of
constraint forces $\avs\du$. Thus, denoting the norm of the constraint
force $\afc\in\avs\du$ by $\norm{\afc}$, the properties of norm
require that $\norm{\afc}\ges0$, $\norm{a\afc}=\abs a\norm{\afc}$,
$\norm{\afc_{1}+\afc_{2}}\les\norm{\afc_{1}}+\norm{\afc_{2}}$, for
all $\afc,\afc_{1},\afc_{2}\in\avs\du$, $a\in\reals$, and that $\norm g=0$
if and only if $\afc=0$.

Thus, the grasping optimization problem may be formulated as follows.
Given $\jac:\vs\to\avs$, and the force $\fc\in\vs\du$ acting on
the object, find
\begin{equation}
\opt_{\fc}:=\inf\{\norm{\afc}\mid\afc\in\avs\du,\,\jac\du(\afc)=\fc\}.
\end{equation}

We show below some properties of the optimum $\opt_{\fc}$ below and
in particular, we show the existence of a minimizer $\afc_{0}$ (not
necessarily unique) such that $\opt_{\fc}=\norm{\afc_{0}}$. For this
purpose and in order to make the text self-contained, we review below
some properties of normed vector spaces.

\subsection{Primal norms and dual norms}

We consider a generic vector space $\vs$ with a norm $\norm{\cdot}$.
(The definitions and results will be applied to both $\vs$ and $\avs$
above.) Since this is our starting point, we will refer to $\vs$
as the \emph{primal vector space }and to the prescribed norm as the
\emph{primal norm}.

The norm on the vector space $\vs$ induces naturally a norm, the
\emph{dual norm} on the dual space $\vs\du$. The dual norm of an
element $\fc\in\vs\du$ is defined by 
\begin{equation}
\norm{\fc}=\sup_{\v\neq0}\frac{\abs{\fc(\v)}}{\norm{\v}}=\sup_{\norm{\v}=1}\abs{\fc(\v)}.\label{eq:dual_norm}
\end{equation}
Note that we use the same notation for both primal and dual norms
when no confusion may arise. 

It is recalled that for the finite-dimensional case $\dimension\vs\du=\dimension\vs$
and for any $v\in\vs$
\begin{equation}
\norm v=\sup_{\fc\ne0}\frac{\abs{\fc(v)}}{\norm{\fc}}=\sup_{\norm{\fc}=1}\abs{\fc(v)}.\label{eq:primal_norm}
\end{equation}

For the optimization problem considered, it is natural to assume that
the norm for the constraint forces, which one wishes to minimize,
is chosen on the basis of considerations related to the design of
the gripping mechanism. This implies that in for the optimization
problem, one starts with the dual norm in $\avs\du$. Thus, Equation
(\ref{eq:primal_norm}) provides the corresponding primal norm on
$\avs$.

\subsection{Extension of covectors and static indeterminacy}

Consider the following diagram\begin{equation}\label{eq:diagram}
\begin{xy}
(40,25)*+{\image J} ="TS";
(-10,25)*+{\reals} = "TQ";
(15,25)*+{\vs} ="jTS";
(65,25)*+{\avs}="U";
{\ar@{<-}@/^{2pc}/^{\afc_0=\fc\comp \wh J^{-1}} "TQ"; "TS"};
{\ar@{<-}^{\fc} "TQ"; "jTS"};
{\ar@{<-}^{\wh\jac^{-1}} "jTS"; "TS"};
{\ar@{->}@/_{1pc}/_{\wh\jac} "jTS"; "TS"};
{\ar@{->}^{\incl} "TS"; "U"};
\end{xy}
\end{equation}where 
\begin{equation}
\wh{\jac}:\vs\to\image\jac,\qquad\text{and}\qquad\incl:\image\jac\to\avs\label{eq:def_J-hat}
\end{equation}
are the mapping between $\vs$ and its image under $\jac$, and the
natural inclusion of the vector subspace, so that $\jac=\incl\comp\wh{\jac}$.
Since $\jac$ is injective, $\wh{\jac}:\vs\to\image\jac$ is an isomorphism,
and $\wh J^{-1}:\image\jac\to\vs$ is a well-defined isomorphism.
A generalized force, $\fc\in\vs\du$, acting on the object, induces
a unique element $\afc_{0}:=\fc\comp\wh{\jac}^{-1}\in(\image\jac)\du$.
Conversely, any $\afc_{0}\in(\image\jac)\du$ determined a unique
$\fc\in\vs\du$ by $\fc=\afc_{0}\comp\wh{\jac}$, or $\fc(v)=\afc_{0}(\jac(v))$,
for all $v\in\vs$. In other words, 
\begin{equation}
\wh J^{*}:(\image\jac)\du\tto\vs\du
\end{equation}
is an isomorphism.

Note however that for a given $\fc\in\vs\du$, $\afc_{0}=\wh{\jac}\du(\fc)$
is not a constraint force. It is an element of $(\image\jac)\du$
and not an element of $\avs\du$. In other words, we know how $\afc_{0}$
acts only on all compatible velocities of the constraints and not
on all generalized constraint velocities. While a constraint force
$\afc\in\avs\du$ may be restricted uniquely to an element $\afc_{0}\in(\image\jac)\du$
by $\afc_{0}(u)=\afc(u)$, for the situation where $m>n$, the extension
of $\afc_{0}$ to the whole of $\avs$ is not unique. This observation
is another aspect of static indeterminacy. It follows that our optimization
problem is equivalent to minimizing the extension $\afc\in\avs\du$
of $\afc_{0}=\wh{\jac}\du(\fc)\in(\image\jac)\du$ from $\image\jac$
to $\avs$.

\subsection{\label{subsec:HB-consequences}Norm-minimizing extensions and the
Hahn-Banach theorem}

In view of the conclusion of the preceding paragraph, we consider
norm-minimizing extensions of covectors. Thus, consider $\svs=\image\jac\subset\avs$.
(The same construction applies to any other vector subspace $\svs\subset\avs$.)
It is assumed that dual norms are given as above on $\avs$ and $\avs\du$.
The norm on $\avs$ induces, by restriction, a norm on $\svs$. In
turn, a norm on $\svs\du$ is induced by setting
\begin{equation}
\norm{\afc_{0}}=\sup_{\vf\neq0}\frac{\abs{\afc_{0}(\vf)}}{\norm{\vf}}=\sup_{\norm{\vf}=1}\abs{\afc_{0}(\vf)},\qquad\vf\in\svs.
\end{equation}

It is noted that in case $\afc_{0}\in\svs\du$ is the restriction
of some constraint force $\afc\in\avs\du$, then,
\begin{equation}
\norm{\afc_{0}}=\sup_{\substack{\vf\neq0\\
\vf\in\svs
}
}\frac{\abs{\afc_{0}(\vf)}}{\norm{\vf}}=\sup_{\substack{\vf\neq0\\
\vf\in\svs
}
}\frac{\abs{\afc(\vf)}}{\norm{\vf}}\les\sup_{\substack{\avf\neq0\\
\avf\in\avs
}
}\frac{\abs{\afc(\avf)}}{\norm{\avf}}=\norm{\afc}.
\end{equation}
Thus, the norm of a constraint force is bounded from below by the
norm of its restriction to compatible velocities of the constraints.
The isomorphism $\wh{\jac}\du$ implies that the optimum $\opt_{\fc}$
is bounded below by the norm of $\afc_{0}=\wh{\jac}\du(\fc)$.

The Hahn-Banach theorem, in fact, the theorem proved by Helly for
the finite dimensional case in \cite{Helly1912}, asserts that any
$\afc_{0}\in\svs\du$ may be extended to some $\afc_{\textrm{H}}\in\avs\du$
without increasings its norm. The proof of the theorem is given in
Appendix \ref{sec:H-B} both for the sake of completeness, and, in
particular, because we will use the construction in the following
examples.

It is concluded that the lower bound $\norm{\afc_{0}}$ on $\norm{\afc}$
is attained by some $\afc_{\mathrm{H}}\in\avs\du$ and the optimal
extension is given by 
\begin{equation}
\norm{\afc_{0}}=\inf\{\norm{\afc}\mid\afc\in\avs\du,\afc(\vf)=\afc_{0}(\vf),\,\forall\vf\in\svs\}=\norm{\afc_{\mathrm{H}}}.
\end{equation}

\subsection{\label{subsec:Optimum}Norm minimizing grasping}

Once optimal extensions are at our disposal, we may continue our study
of optimal grasping.

Let a norm be given on the vector space $\vs:=T_{\conf}\csa$ of generalized
velocities of the object.
\begin{prop}
\label{prop:Optimal}Using the notation introduced above, for a fixed
configuration of the object, $\conf\in\csa$, let 
\begin{equation}
\jac:=T_{\conf}\fk:\vs:=T_{\conf}\csa\tto\avs:=T_{\fk(\conf)}\csb,
\end{equation}
be the tangent of the constraint mapping.
\begin{enumerate}
\item Assume that $\jac$ is injective. Then, given a force $\fc\in\vs\du$
on the object there exists some $\afc\in\avs\du$ that satisfy the
equilibrium equations, that is
\[
\fc=\jac\du(\afc).
\]
\item Assume that $\vs\du$ and $\avs\du$ are normed and for a given $\fc\in\vs\du$
let the optimal grasping problem be defined as
\begin{equation}
\opt_{\fc}:=\inf\{\norm{\afc}\mid\afc\in\avs\du,\,\jac\du(\afc)=\fc\}.
\end{equation}
Then, 
\begin{equation}
\opt_{\fc}=\sup_{v\in\vs}\frac{\fc(v)}{\norm{\jac(v)}},\label{eq:Optimal_Reactions}
\end{equation}
where the norms on  $\avs$ is the primal norm for that given on
 $\avs\du$. In particular, for the case where $\vs$ is a normed
space, $\vs\du$ is equipped with the dual norm, and $\jac$ is norm
preserving,
\begin{equation}
\opt_{\fc}=\norm{\fc}.
\end{equation}
\item The optimal grasping is attained for some $\afc_{\mathrm{H}}\in\avs\du$,
that is, there is some $\afc_{\mathrm{H}}\in\avs\du$, such that
\begin{equation}
\opt_{\fc}=\norm{\afc_{\mathrm{H}}}.
\end{equation}
\item The optimal grasping force is positive homogeneous, that is, for every
$a\in\reals$, 
\begin{equation}
\opt_{a\fc}=\abs a\opt_{\fc}.\label{eq:Optimum_homogeneous}
\end{equation}
In particular,
\begin{equation}
\opt_{\fc/\opt_{\fc}}=1.
\end{equation}
\end{enumerate}
\end{prop}

\begin{proof}
Using the foregoing notation and observations we proceed as follows.

Let $\wh{\jac}:\vs\to\image\jac$ be the restricted isomorphism. Then,
$\afc_{0}\in(\image\jac)\du$, given by $\afc_{0}=\fc\comp\wh{\jac}^{-1}$
satisfies 
\begin{equation}
f=\afc_{0}\comp\wh{\jac}=\jac\du(\afc_{0}),\qquad\text{or explicitly,}\qquad f(v)=\afc_{0}(\jac(v))\fall v\in\vs.
\end{equation}
The covector $\afc_{0}\in(\image\jac)\du$ may be extended to some
$\afc\in\avs^{*}$ satisfying $\afc(\vf)=\afc_{0}(\vf)$ for all $\vf\in\image\jac$.
Hence,
\begin{equation}
\fc(v)=\afc(\jac(v))\fall v\in\vs
\end{equation}
and we conclude that $\fc=\jac\du(\afc)$. This proves the first assertion
above.

By definition,
\begin{equation}
\norm{\afc_{0}}:=\sup_{\vf\in\image\jac}\frac{\afc_{0}(\vf)}{\norm{\vf}}.
\end{equation}
Since $\wh{\jac}$ is an isomorphism, we may write
\begin{equation}
\begin{split}\norm{\afc_{0}} & =\sup_{v\in\vs}\frac{\afc_{0}(\jac(v))}{\norm{\jac(v)}},\\
 & =\sup_{v\in\vs}\frac{\fc(v)}{\norm{\jac(v)}},
\end{split}
\end{equation}
where it is observed that $\norm{\afc_{0}}$ is uniquely determined
by $\fc$. As a consequence of the Hahn-Banach theorem as discussed
in Section \ref{subsec:HB-consequences}, there is some $\afc_{\mathrm{H}}\in\avs\du$
extending $\afc_{0}$ such that $\norm{\afc_{\mathrm{H}}}=\norm{\afc_{0}}$---the
minimal norm of all possible extensions. Thus, 
\begin{equation}
\opt_{\fc}=\norm{\afc_{\mathrm{H}}}=\sup_{v\in\vs}\frac{\fc(v)}{\norm{\jac(v)}},
\end{equation}
which proves (2) and (3) above.

Evidently, if $\jac$ is norm preserving so that $\norm{\jac(v)}=\norm v$,
one has
\begin{equation}
\opt_{\fc}=\sup_{v\in\vs}\frac{\fc(v)}{\norm v}=\norm{\fc}.
\end{equation}

Assertion (4) follows immediately from Equation (\ref{eq:Optimal_Reactions}).
\end{proof}
Given a force $\fc$ acting on the object, the \textit{sensitivity
ratio for $\fc$, $R_{\fc}$, }is defined by
\begin{equation}
R_{\fc}:=\inf_{\substack{\fc=\jac\du(\afc)\\
\afc\in\avs\du
}
}\frac{\norm{\afc}}{\norm{\fc}}.
\end{equation}
It immediately follows from the foregoing result that 
\begin{equation}
R_{\fc}=\frac{\opt_{\fc}}{\norm{\fc}}=\frac{1}{\norm{\fc}}\sup_{v\in\vs}\frac{\fc(v)}{\norm{\jac(v)}}.\label{eq:sensitivity_ratio}
\end{equation}
From the practical point of view, $R_{\fc}$ indicates how much the
magnitude of $\fc$ is amplified for optimal grasping forces.
\begin{rem}
\label{rem:Rf_norm-indep}Equation (\ref{eq:Optimum_homogeneous})
implies that the sensitivity ratio for $\fc$ does not depend on $\norm{\fc}$.
It depends only on its normalized counterpart, $\fc/\norm{\fc}$,
that is, for any $a\in\reals^{+}$, 
\begin{equation}
R_{a\fc}=R_{\fc}=R_{\fc/\norm{\fc}}.
\end{equation}
\end{rem}

It is natural to refer to $1/R_{\fc}$ as the \textit{attenuation
of the force $\fc$.} One has,
\begin{equation}
\begin{split}\frac{1}{R_{\fc}} & =\frac{\norm{\fc}}{\opt_{\fc}},\\
 & =\frac{\norm{\fc}}{\sup_{v\in\vs}\{\fc(v)/\norm{\jac(v)}\}},\\
 & =\norm{\fc}\inf_{v\in\vs}\frac{\norm{\jac(v)}}{\fc(v)}.
\end{split}
\label{eq:attenuation}
\end{equation}

\subsection{\label{subsec:Example:wafer}Example: grasping a 2-dimensional object
at two points}

Figure \ref{fig:wafer1} illustrates two robotic hands holding a silicon
wafer at two points at the ends of a vertical diameter. The offset
of the weight vector from the middle of the vertical diameter is denoted
by $a$. In a 2-dimensional analysis, it is assumed that each gripper
can exert two force components in the plane. Thus, using the generalized
coordinates $(x,y,\theta)$ we may identify the configuration space
$\csa$ with $\reals^{2}\times S^{1}$ and the configuration shown
is $\conf=(0,0,0)$. The configuration space of the constrained points
is $\csb=\reals^{4}=\{(x_{1},y_{1},x_{2},y_{2})\}$ (see Figure \ref{fig:wafer1}).

\begin{figure}
\begin{centering}
\includegraphics[scale=0.6]{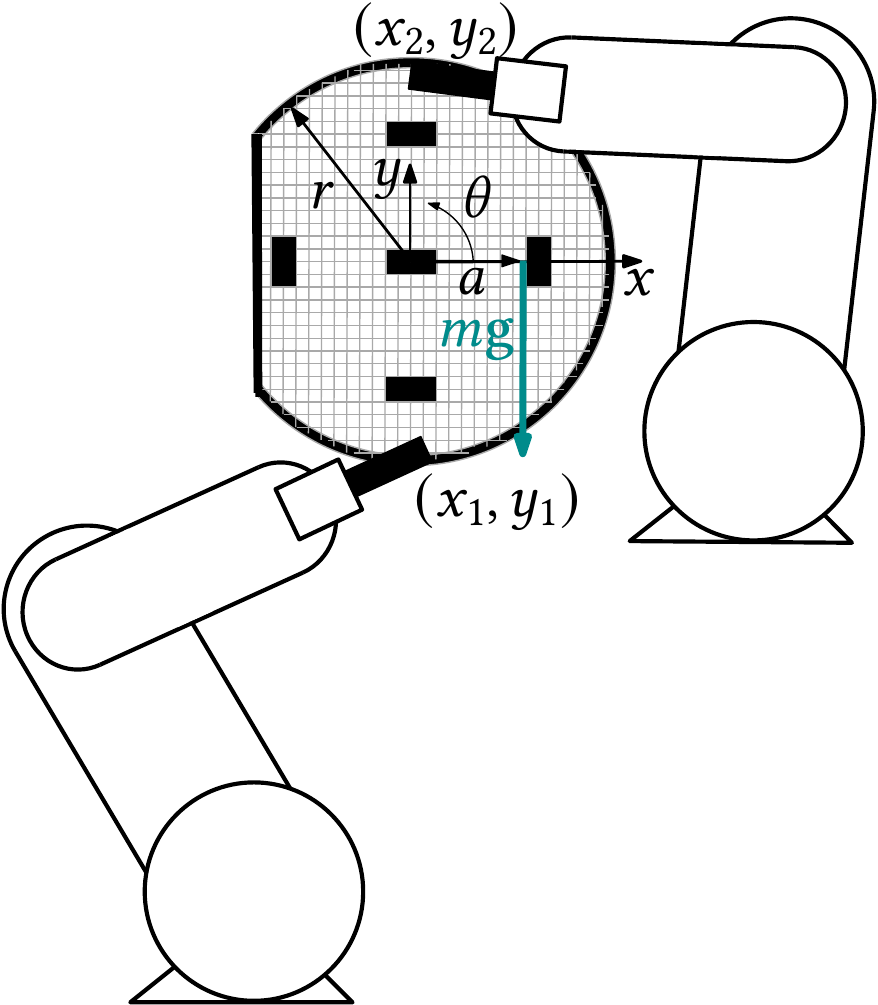}
\par\end{centering}
\caption{\label{fig:wafer1}Grasping a silicon wafer at two points}
\end{figure}
The constraint mapping $\fk:\csa\to\csb$ may be written explicitly
as
\begin{equation}
\fk(x,y,\theta)=\left(\begin{gathered}x+r\sin\theta\\
y+r(1-\cos\theta)\\
x-r\sin\theta\\
y-r(1-\cos\theta)
\end{gathered}
\right),
\end{equation}
and differentiating at the configuration $(0,0,0)$ gives
\begin{equation}
\jac=T_{\conf}\fk=\left(\begin{array}{ccc}
1 & 0 & r\\
0 & 1 & 0\\
1 & 0 & -r\\
0 & 1 & 0
\end{array}\right),\qquad\jac(v_{x},v_{y},\omega)=\left(\begin{gathered}v_{x}+r\omega\\
v_{y}\\
v_{x}-r\omega\\
v_{y}
\end{gathered}
\right),\qquad\jac\du=\left(\begin{array}{cccc}
1 & 0 & 1 & 0\\
0 & 1 & 0 & 1\\
r & 0 & -r & 0
\end{array}\right),\label{eq:jac1}
\end{equation}
where $(v_{x},v_{y},\omega)\in\reals^{3}$ represent an element in
$T_{\conf}\csa$ and the dual mapping is obtained by transposition.

The constraint force vector is of the form $\afc=(\afc_{x_{1}},\afc_{y_{1}},\afc_{x_{2}},\afc_{y_{2}})$,
and we consider first the sup norm so that
\begin{equation}
\norm{\afc}=\max\{\afc_{x_{1}},\afc_{y_{1}},\afc_{x_{2}},\afc_{y_{2}}\}.
\end{equation}
It follows that the primal norm on $\avs:=T_{\fk(\conf)}\csb\isom\reals^{4}$
is 
\begin{equation}
\norm{\avf}=\abs{\avf_{x_{1}}}+\abs{\avf_{y_{1}}}+\abs{\avf_{x_{2}}}+\abs{\avf_{y_{2}}}.
\end{equation}

\subsubsection{The case $a=0$, no offset}

For simplicity, we consider first the case where $a=0$ so that the
weight acts in the center. By Equation (\ref{eq:Optimal_Reactions}),
the optimum is given by

\begin{gather}
\opt_{\fc}=\sup_{\v\in\vs}\frac{\fc(\v)}{\norm{\jac(\v)}}=\sup_{\v\in\vs}\frac{-m\g v_{y}}{\abs{v_{x}+r\omega}+\abs{v_{y}}+\abs{v_{x}-r\omega}+\abs{v_{y}}}.
\end{gather}
Let $\ga,\gb\in\reals$ satisfy $v_{x}=\ga v_{y}$ and $r\go=\gb v_{y}$,
the optimum may be rewritten as
\begin{equation}
\opt_{\fc}=\sup_{\alpha,\beta}\frac{m\g}{\abs{\alpha+\beta}+\abs{\alpha-\beta}+2}=\frac{m\g}{2}.
\end{equation}

Next, we construct the optimal grasping forces having this norm. The
covector $\afc_{0}=\fc\comp\wh{\jac}^{-1}\in(\image\jac)\du$ satisfies,
by Equation (\ref{eq:jac1}),
\begin{equation}
\afc_{0}(\avf)=\fc\comp\wh{\jac}^{-1}(\avf_{x_{1}},\avf_{y_{1}},\avf_{x_{2}},\avf_{y_{2}})=-m\g v_{y}=-m\g\avf_{y_{1}}.
\end{equation}
In order to extend $\afc_{0}$ to $\afc_{\mathrm{H}}$ as in Proposition
\ref{prop:Optimal}, we use the Helly construction of the Hahn-Banach
theorem, as in Appendix \ref{sec:H-B}. As a basis for $\image\jac$
we use
\begin{equation}
\left\{ \left(\begin{gathered}0\\
1\\
0\\
1
\end{gathered}
\right),\left(\begin{gathered}1\\
0\\
1\\
0
\end{gathered}
\right),\left(\begin{gathered}1\\
0\\
-1\\
0
\end{gathered}
\right)\right\} ,
\end{equation}
and complete it to a basis of $\avs$ by adding the vector $(0,0,0,1)\in\reals^{4}$.

By Equation (\ref{eq:sup-inf H-B}), we calculate
\begin{equation}
\sup_{\image\jac}\left\{ \afc_{0}(\avf_{x_{1}},\avf_{y_{1}},\avf_{x_{2}},\avf_{y_{2}})-\norm{\afc_{0}}\norm{(\avf_{x_{1}},\avf_{y_{1}},\avf_{x_{2}},\avf_{y_{2}})-(0,0,0,1)}\right\} 
\end{equation}
and
\begin{equation}
\inf_{\image\jac}\left\{ \norm{\afc_{0}}\left\Vert (\avf,\avf_{y_{1}},\avf_{x_{2}},\avf_{y_{2}})+(0,0,0,1)\right\Vert -\afc_{0}(\avf_{x_{1}},\avf_{y_{1}},\avf_{x_{2}},\avf_{y_{2}})\right\} 
\end{equation}
and we should set $\afc_{H}(0,0,0,1)$ to be between those values.
\begin{multline}
\sup_{\image\jac}\left\{ \afc_{0}\left(\begin{gathered}\avf_{x_{1}}\\
\avf_{y_{1}}\\
\avf_{x_{2}}\\
\avf_{y_{2}}
\end{gathered}
\right)-\norm{\afc_{0}}\left\Vert \left(\begin{gathered}\avf_{x_{1}}\\
\avf_{y_{1}}\\
\avf_{x_{2}}\\
\avf_{y_{2}}
\end{gathered}
\right)-\left(\begin{gathered}0\\
0\\
0\\
1
\end{gathered}
\right)\right\Vert \right\} \\
\begin{aligned} & =\sup\left\{ -m\g\avf_{y_{1}}-\frac{m\g}{2}\left(\abs{\avf_{x_{1}}}+\abs{\avf_{y_{1}}}+\abs{\avf_{x_{2}}}+\abs{\avf_{y_{2}}-1}\right)\right\} ,\\
 & =m\g\sup\left\{ -\avf_{y_{1}}-\frac{1}{2}\abs{\avf_{x_{1}}}-\frac{1}{2}\abs{\avf_{y_{1}}}-\frac{1}{2}\abs{\avf_{x_{2}}}-\frac{1}{2}\abs{\avf_{y_{2}}-1}\right\} ,\\
 & =m\g\sup\left\{ -\avf_{y_{1}}-\frac{1}{2}\abs{\avf_{y_{1}}}-\frac{1}{2}\abs{\avf_{y_{1}}-1}\right\} ,\\
 & =-\frac{m\g}{2},
\end{aligned}
\end{multline}
where in the second line we used $\norm{\afc_{0}}=\opt_{\fc}=m\g/2$,
in the third line we omitted non-positive term and used $\avf_{y_{1}}=\avf_{y_{2}}$
in $\image\jac$. The last line is obtained by considering all possible
values of $\avf_{y_{1}}$.

Similarly,
\begin{multline}
\inf_{\image\jac}\left\{ \norm{\afc_{0}}\left\Vert \left(\begin{gathered}\avf_{x_{1}}\\
\avf_{y_{1}}\\
\avf_{x_{2}}\\
\avf_{y_{2}}
\end{gathered}
\right)+\left(\begin{gathered}0\\
0\\
0\\
1
\end{gathered}
\right)\right\Vert -\afc_{0}\left(\begin{gathered}\avf_{x_{1}}\\
\avf_{y_{1}}\\
\avf_{x_{2}}\\
\avf_{y_{2}}
\end{gathered}
\right)\right\} \\
\begin{aligned} & =\inf\left\{ \frac{m\g}{2}\left(\abs{\avf_{x_{1}}}+\abs{\avf_{y_{1}}}+\abs{\avf_{x_{2}}}+\abs{\avf_{y_{2}}+1}\right)+m\g\avf_{y_{1}}\right\} ,\\
 & =m\g\inf\left\{ \frac{1}{2}\left(\abs{\avf_{x_{1}}}+\abs{\avf_{y_{1}}}+\abs{\avf_{x_{2}}}+\abs{\avf_{y_{2}}+1}\right)+\avf_{y_{1}}\right\} ,\\
 & =m\g\inf\left\{ \frac{1}{2}\left(\abs{\avf_{y_{1}}}+\abs{\avf_{y_{1}}+1}\right)+\avf_{y_{1}}\right\} ,\\
 & =-\frac{m\g}{2}.
\end{aligned}
\end{multline}

Therefore, we must choose
\begin{equation}
\afc_{\mathrm{H}}(0,0,0,1)=-\frac{m\g}{2}.
\end{equation}
Finally, we may compute $\afc_{\mathrm{H}}$ by applying it to each
base vector to obtain
\begin{equation}
\begin{split}\afc_{\mathrm{H}}\left(\begin{gathered}\avf_{x_{1}}\\
\avf_{y_{1}}\\
\avf_{x_{2}}\\
\avf_{y_{2}}
\end{gathered}
\right) & =\avf_{y_{1}}\afc_{\mathrm{H}}\left(\begin{gathered}0\\
1\\
0\\
1
\end{gathered}
\right)+\frac{\avf_{x_{2}}+\avf_{x_{1}}}{2}\afc_{\mathrm{H}}\left(\begin{gathered}1\\
0\\
1\\
0
\end{gathered}
\right)+\frac{\avf_{x_{1}}-\avf_{x_{2}}}{2}\afc_{\mathrm{H}}\left(\begin{gathered}1\\
0\\
-1\\
0
\end{gathered}
\right)+(\avf_{y_{2}}-\avf_{y_{1}})\afc_{\mathrm{H}}\left(\begin{gathered}0\\
0\\
0\\
1
\end{gathered}
\right),\\
 & =-m\g\avf_{y_{1}}-(\avf_{y_{2}}-\avf_{y_{1}})\frac{m\g}{2},\\
 & =-\frac{m\g}{2}\avf_{y_{1}}-\frac{m\g}{2}\avf_{y_{2}},
\end{split}
\end{equation}
so that 
\begin{equation}
\afc_{\mathrm{H}}=(\begin{array}{cccc}
0 & -m\g/2 & 0 & -m\g/2\end{array}).
\end{equation}

Thus, the general procedure proposed yields indeed the expected result.

\subsubsection{The case $a\protect\ne0$, eccentric loading}

Next, we consider the eccentric situation where $a\ne0$ and $a\les r$
as illustrated in Figure \ref{fig:wafer1}. Evidently, the generalized
external force on the object is $\fc=(0,-m\g,-m\g a)$. Using the
same norms as above, Equation (\ref{eq:Optimal_Reactions}) implies
that 
\begin{equation}
\opt_{\fc}=\sup_{\v\in\vs}\frac{\fc(\v)}{\norm{\jac(\v)}}=\sup_{\v\in\vs}\frac{-m\g v_{y}-am\g\omega}{\abs{v_{x}+r\omega}+\abs{v_{y}}+\abs{v_{x}-r\omega}+\abs{v_{y}}}.
\end{equation}
Once again, we define $\ga,\gb\in\reals$, by $v_{x}=\alpha v_{y},\,\omega=\beta v_{y}$,
so that the expression for the optimum assumes the form

\begin{equation}
\begin{split}\opt_{\fc} & =\sup_{\alpha,\beta}\frac{m\g(1+a\beta)}{\abs{\alpha+r\beta}+\abs{\alpha-r\beta}+2},\\
 & =\sup_{\beta}\frac{m\g(1+a\beta)}{2r\abs{\beta}+2},
\end{split}
\end{equation}
where the triangle inequality was used in order to arrive at the second
line.

Setting 
\begin{equation}
h(\gb):=\frac{m\g(1+a\beta)}{2r\abs{\beta}+2},
\end{equation}
analyzing $h$ by differentiation for the cases $\gb\ges0$ and $\gb\les0$,
and using the assumption that $a\les r$, one concludes that the supremum
above is attained for $\gb=0$ as for the case $a=0$. Thus,
\[
\opt_{\fc}=\sup_{\beta}h(\beta)=\frac{m\g}{2}.
\]

To construct optimal grasping forces, we note first that $\afc_{0}:\image\jac\to\reals$
is given by 
\begin{equation}
\begin{split}\afc_{0}(\avf_{x_{1}},\avf_{y_{1}},\avf_{x_{2}},\avf_{y_{2}}) & =\fc\comp\wh{\jac}^{-1}(\avf_{x_{1}},\avf_{y_{1}},\avf_{x_{2}},\avf_{y_{2}}),\\
 & =-m\g\avf_{y_{1}}-\frac{\avf_{x_{1}}-\avf_{x_{2}}}{2r}am\g.
\end{split}
\end{equation}
Next, we use a slightly different method, also based on the Helly
construction, for the extension of the covector $\afc_{0}\in(\image\jac)\du$
to some $\afc\h\in\avs\du$. Using the same basis for the grasping
forces as above, we set $\afc\h(0,0,0,1)=c$, where $c\in\reals$
is yet undetermined. Thus,

\begin{equation}
\begin{split}\afc_{\mathrm{H}}\left(\begin{gathered}\avf_{x_{1}}\\
\avf_{y_{1}}\\
\avf_{x_{2}}\\
\avf_{y_{2}}
\end{gathered}
\right) & =\avf_{y_{1}}\afc_{\mathrm{H}}\left(\begin{gathered}0\\
1\\
0\\
1
\end{gathered}
\right)+\frac{\avf_{x_{2}}+\avf_{x_{1}}}{2}\afc_{\mathrm{H}}\left(\begin{gathered}1\\
0\\
1\\
0
\end{gathered}
\right)+\frac{\avf_{x_{1}}-\avf_{x_{2}}}{2}\afc_{\mathrm{H}}\left(\begin{gathered}1\\
0\\
-1\\
0
\end{gathered}
\right)+(\avf_{y_{2}}-\avf_{y_{1}})\afc_{\mathrm{H}}\left(\begin{gathered}0\\
0\\
0\\
1
\end{gathered}
\right),\\
 & =-m\g\avf_{y_{1}}-\frac{\avf_{x_{1}}-\avf_{x_{2}}}{2r}am\g+(\avf_{y_{2}}-\avf_{y_{1}})c,\\
 & =-\frac{am\g}{2r}\avf_{x_{1}}+(-m\g-c)\avf_{y_{1}}+\frac{am\g}{2r}\avf_{x_{2}}+\avf_{y_{2}}c.
\end{split}
\label{eq:gh_for_wafer}
\end{equation}
Since, by the Hahn-Banach theorem $\norm{\afc\h}=\norm{\afc_{0}}=m\g/2$,
one has
\begin{equation}
\max\left\{ \abs{\frac{am\g}{2r}},\abs{-m\g-c},\abs{\frac{-am\g}{2r}},\abs c\right\} =\frac{m\g}{2}.
\end{equation}
We now make an ansatz that $c=-m\g/2$. This will imply that
\begin{equation}
\afc\h=\left(-\frac{am\g}{2r},-\frac{m\g}{2},\frac{am\g}{2r},\frac{m\g}{2}\right)
\end{equation}
which justifies our ansatz as long as $a\les r$. (One can follow
the analogous reasoning for the case $a>r$.)

It is observed that in the Helly construction in Appendix \ref{sec:H-B},
the number $c$ that determines the extension uniquely, might be,
in general, in an interval in $\reals$. However, in our case, this
interval collapses to one point. Hence, the extension is unique in
these examples.

\subsubsection{Using the Euclidean norm for the grasping forces}

A more realistic norm for the space of grasping forces $\avs\du=\{(\afc_{x_{1}},\afc_{y_{1}},\afc_{x_{2}},\afc_{y_{2}})\}$
is
\begin{equation}
\norm g=\max\left\{ \sqrt{g_{x_{1}}^{2}+g_{y_{1}}^{2}},\sqrt{g_{x_{2}}^{2}+g_{y_{2}}^{2}}\right\} .
\end{equation}
The corresponding primal on $\avs$ norm is given by

\begin{equation}
\norm{\avf}=\sqrt{\avf_{x_{1}}^{2}+\avf_{y_{1}}^{2}}+\sqrt{\avf_{x_{2}}^{2}+\avf_{y_{2}}^{2}}.
\end{equation}

Equation (\ref{eq:Optimal_Reactions}) assumes the form
\begin{gather}
\opt_{\fc}=\sup_{\v\in\vs}\frac{\fc(\v)}{\norm{\jac(\v)}}=\sup_{\v\in\vs}\frac{-m\g v_{y}-am\g\omega}{\sqrt{(v_{x}+r\omega)^{2}+v_{y}^{2}}+\sqrt{(v_{x}-r\omega)^{2}+v_{y}^{2}}},
\end{gather}
and setting $v_{x}=\alpha v_{y},\omega=\beta v_{y}$, we get

\begin{gather}
\opt_{\fc}=\sup_{\alpha,\beta}\frac{m\g(1+a\beta)}{\sqrt{(\alpha+r\beta)^{2}+1}+\sqrt{(\alpha-r\beta)^{2}+1}}=\sup_{\beta}\frac{m\g(1+a\beta)}{2\sqrt{r^{2}\beta^{2}+1}}.
\end{gather}
By differentiating the function
\begin{equation}
h(\gb):=\frac{m\g(1+a\beta)}{2\sqrt{r^{2}\beta^{2}+1}},
\end{equation}
one obtains that the maximum is attained at $\gb=a/r^{2}$, and that
\begin{equation}
\opt_{\fc}=\sup h(\beta)=\frac{m\g\sqrt{a^{2}+r^{2}}}{2r}.
\end{equation}
Hence, Equation (\ref{eq:gh_for_wafer}) assumes the form
\begin{equation}
\max\left\{ \sqrt{\left(\frac{am\g}{2r}\right){}^{2}+(-m\g-c)^{2}},\sqrt{\left(\frac{-am\g}{2r}\right){}^{2}+c^{2}}\right\} =\frac{m\g\sqrt{a^{2}+r^{2}}}{2r}.
\end{equation}
We make the ansatz

\begin{equation}
\sqrt{\left(\frac{-am\g}{2r}\right){}^{2}+c^{2}}=\frac{m\g\sqrt{a^{2}+r^{2}}}{2r},\qquad c=\pm\frac{m\g}{2},
\end{equation}
which indeed gives the maximum above. The corresponding optimal grasping
forces are given by 
\begin{equation}
\afc\h=\left(-\frac{am\g}{2r},-\frac{m\g}{2},\frac{am\g}{2r},\frac{m\g}{2}\right)
\end{equation}
as for the supremum norm.

\subsection{Example: \label{subsec:Example_Beam}Supporting the weight of a 2-D
body at three points}

Consider the 2-dimensional analysis of rigid body supported by three
robotic hands. Only vertical forces are applied by the grippers, and
we ignore horizontal displacement and forces. See Figure \ref{fig:3points}.

\begin{figure}
\begin{centering}
\includegraphics[scale=0.6]{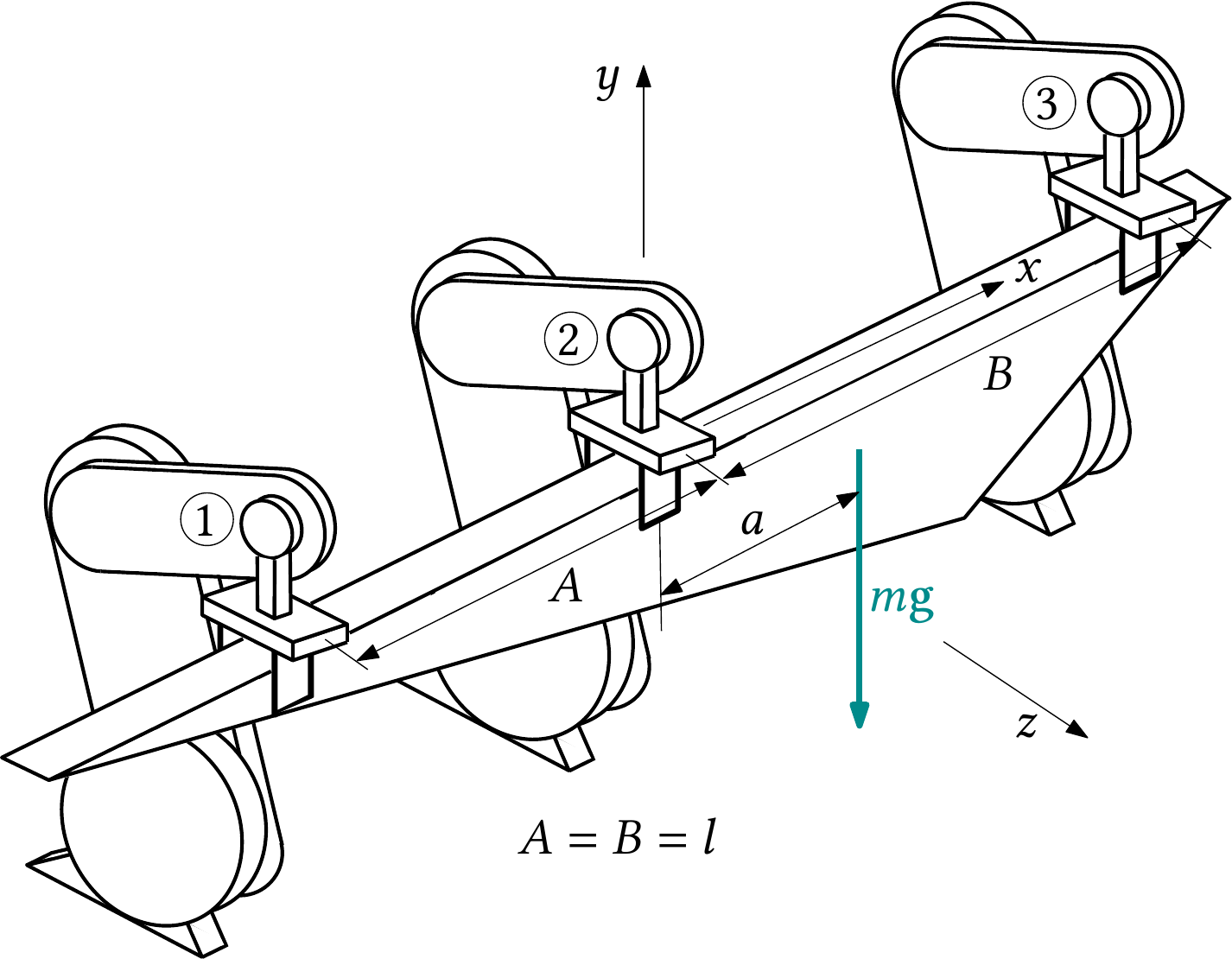}
\par\end{centering}
\caption{\label{fig:3points}Illustrating a 2-dimensional body supported at
three points}
\end{figure}

Here,
\begin{equation}
\vs\isom\{(v_{y},\go)\},\qquad\avs\isom\{(u_{1},u_{2},u_{3})\},
\end{equation}
were $v_{y}$ is the vertical velocity of the point in the body that
is located at the origin in the configuration considered, $\opt$
is the angular velocity of the body, and $\avf_{i}$, are the vertical
velocities of the corresponding grippers as enumerated in the illustration.
Assuming for simplicity that $A=B=l$, one has
\begin{equation}
\jac=T_{\conf}\fk=\left(\begin{array}{cc}
1 & -l\\
1 & 0\\
1 & l
\end{array}\right),\qquad\jac(v_{y},\omega)=\left(\begin{gathered}v_{y}-l\omega\\
v_{y}\\
v_{y}+l\omega
\end{gathered}
\right).
\end{equation}
A generalized force on the body is of the form 
\begin{equation}
\fc=(-m\g,M)=(-m\g,-m\g a),
\end{equation}
where $a$, indicating the location of the center of gravity, serves
to parametrize the moment acting on the body relative to the origin.
A generalized supporting force will be of the form $\afc=(\afc_{1},\afc_{2},\afc_{3})$,
where $\afc_{i}$ indicates the vertical force exerted on the $i$-th
gripper.

Wishing to minimize the maximal vertical force in the grippers, we
use on $\avs\du=\{\afc\}$ the norm
\begin{equation}
\norm{\afc}:=\max\{\afc_{1},\afc_{2},\afc_{3}\}
\end{equation}
for which the primal norm on $\avs$ is
\begin{equation}
\norm{\avf}=\abs{\avf_{1}}+\abs{\avf_{2}}+\abs{\avf_{3}}.
\end{equation}
Using Equation (\ref{eq:Optimal_Reactions}), the optimum is given
by 
\begin{equation}
\opt_{\fc}=\sup_{\v\in\vs}\frac{\fc(\v)}{\norm{\jac(\v)}}=\sup_{\v\in\vs}\frac{-m\g v_{y}-am\g\omega}{\abs{v_{y}-l\omega}+\abs{v_{y}}+\abs{v_{y}+l\omega}}.
\end{equation}
Setting $\ga$ by $\go=\ga v_{y}$, we rewrite,
\begin{equation}
\opt_{\fc}=\sup_{\alpha}\frac{m\g+am\g\alpha}{\abs{1-l\alpha}+1+\abs{1+l\alpha}}.
\end{equation}

Analyzing the function 
\begin{equation}
h(\ga):=\frac{m\g+am\g\alpha}{\abs{1-l\alpha}+1+\abs{1+l\alpha}}
\end{equation}
for the various relative values of $l\ga$, one obtains
\begin{equation}
\opt_{\fc}=\sup_{\alpha}h(\alpha)=\frac{m\g(l+a)}{3l}.
\end{equation}

To compute an optimal generalized grasping force, we first note that
$\fc$ induces the element $\afc_{0}\in(\image\jac)\du$ by 
\[
\afc_{0}(\avf_{1},\avf_{2},\avf_{3})=\fc\comp\wh{\jac}^{-1}(\avf_{1},\avf_{2},\avf_{3})=-m\g\avf_{2}-m\g a\frac{\avf_{3}-\avf_{1}}{2l}.
\]
Next, as a basis for $\image\jac$ we choose $\{(1,1,1),(-1.0,1)\}$,
and we complete it to a basis of $\avs$ by adding the vector $(0,0,1)$.
Following the Helly-Hahn-Banach construction, we set $\afc\h(0,0,1)=c$
for a yet unspecified number $c$. Thus, in terms of $c$,
\begin{equation}
\begin{split}\afc_{\mathrm{H}}\left(\begin{gathered}\avf_{1}\\
\avf_{2}\\
\avf_{3}
\end{gathered}
\right) & =\avf_{1}\afc_{\mathrm{H}}\left(\begin{gathered}1\\
1\\
1
\end{gathered}
\right)+(\avf_{2}-\avf_{1})\afc_{\mathrm{H}}\left(\begin{gathered}-1\\
0\\
1
\end{gathered}
\right)+(\avf_{1}-2\avf_{2}+\avf_{3})\afc_{\mathrm{H}}\left(\begin{gathered}0\\
0\\
1
\end{gathered}
\right),\\
 & =-m\g\avf_{1}-(\avf_{2}-\avf_{1})\frac{m\g a}{l}+(\avf_{1}-2\avf_{2}+\avf_{3}),\\
 & =\left(\frac{m\g a}{l}+c\right)\avf_{1}+\left(-m\g-\frac{m\g a}{l}-2c\right)\avf_{2}+c\avf_{3}.
\end{split}
\label{eq:gh_for_wafer-1}
\end{equation}
It follows that $c$ should satisfy the condition 
\begin{equation}
\max\left\{ \abs{\frac{m\g a}{l}+c},\abs{-m\g-\frac{m\g a}{l}-2c},\abs c\right\} =\frac{m\g(l+a)}{3l},
\end{equation}
which gives the value
\begin{equation}
c=\frac{-m\g(l+a)}{3l}.
\end{equation}
We conclude that the optimal grasping forces are given by
\begin{equation}
\afc\h=\left(\frac{m\g(2a-l)}{3l},\frac{-m\g(l+a)}{3l},\frac{-m\g(l+a)}{3l}\right),
\end{equation}
It is observed that the forces in grippers $\sharp2$ and $\sharp3$
are equal, while the force in gripper $\sharp1$ is different. In
fact, for $a>l/2$, $g_{1}>0$.

\section{\label{sec:Load_Capacity}Bounded Grasping Forces and Worst Case
Loadings}

Conceivably, optimal grasping can be realized if the grasping mechanism
has a force control system. When an external force $\fc$ is acting
on the object and some optimal constraint forces are computed, the
force control system will apply the corresponding grasping forces.
This may be desirable so as to prevent damage to the object and to
limit forces exerted by the gripping mechanism and stresses in it.

In this section we consider the situation where the force-controlled
grasping mechanism can apply all grasping forces $\afc\in\avs\du$
provided that $\norm g\les\yi$ and will not apply forces whose norm
is greater than $\yi$. We will refer to such a situation as \emph{bounded
grasping} and to the given bound, $\yi$, as the \emph{grasping bound}.
In other words, the set of \textit{admissible grasping forces} is
exactly $\cl B_{\yi}(0)$, the closed ball or radius $\yi$ centered
at the origin of $\avs^{*}$. A load $\fc$ acting on the object is
\textit{admissible} if it can be supported by admissible grasping
forces. That is, $\fc$ if admissible if $\fc=\jac^{*}(\afc)$ for
some $\afc\in\cl B_{\yi}(0)$.

\subsection{The Minkowski functional for $K=\protect\jac\protect\du(\protect\cl B_{\protect\yi}(0))$}

The image 
\begin{equation}
K:=\jac\du(\cl B_{\yi}(0))
\end{equation}
 of $\cl B_{\yi}(0)$ under $\jac\du$ consists of the collection
of external forces on the object that may be supported by admissible
grasping forces. Since a ball is a convex set and since the image
of a convex set under a linear mapping is convex, $K$ is a convex
set. It is recalled that for a convex set $K$ whose interior contains
the origin, the Minkowski functional for $K$
\begin{equation}
p:\vs\du\tto\cl{\reals}^{+}
\end{equation}
is defined by
\begin{equation}
p(\fc)=\inf\left\{ r\mid r\in\reals^{+},\,\frac{f}{r}\in K\right\} .
\end{equation}

\begin{figure}
\begin{centering}
\includegraphics[scale=0.6]{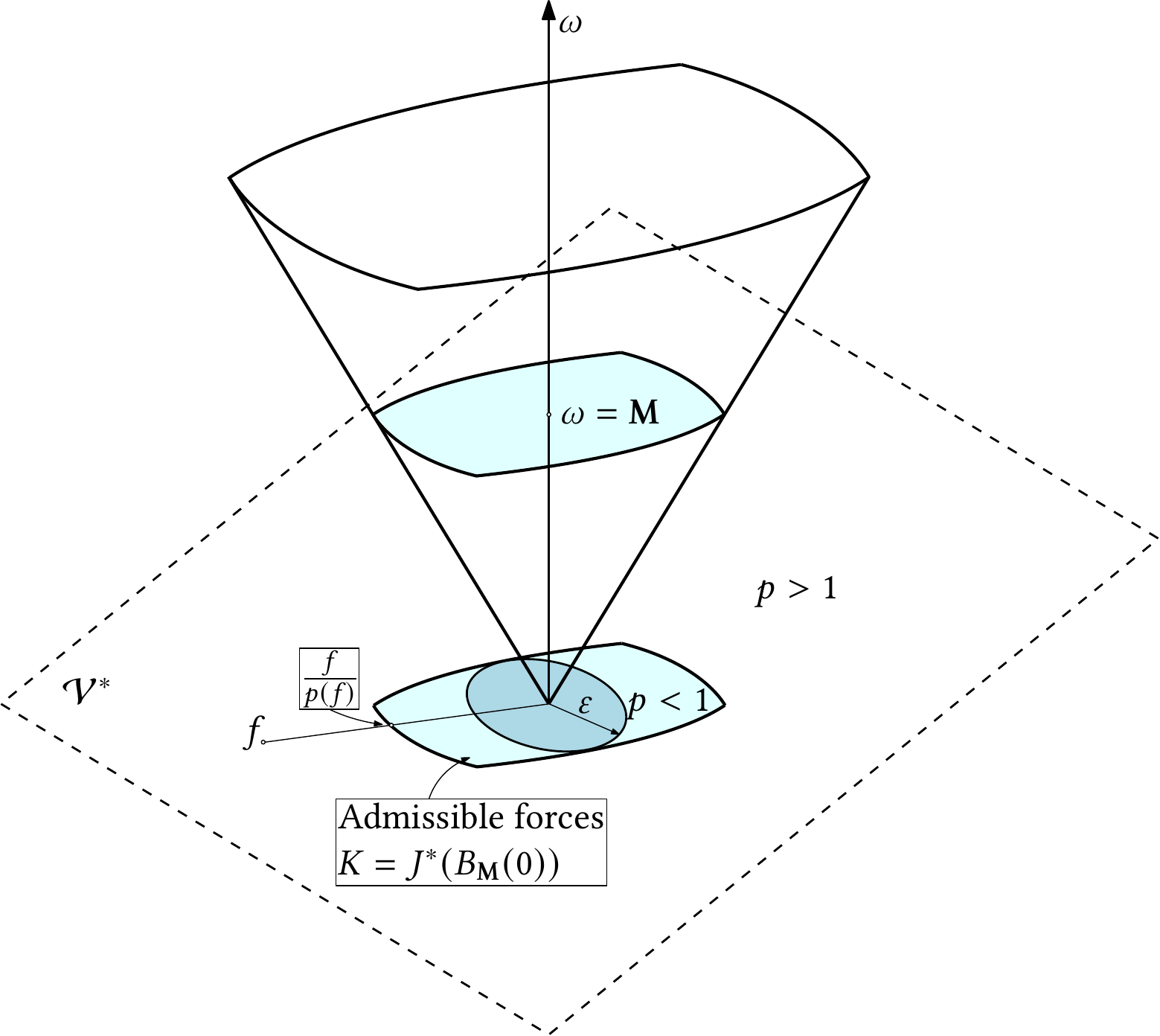}
\par\end{centering}
\caption{\label{fig:Minkowski}Illustrating the Minkowski functional for $K=\protect\jac\protect\du(\protect\cl B_{\protect\yi}(0))$}
\end{figure}
Thus, in order that the Minkowski functional $p$ for $K$ be well
defined, we have to show that there is some $\eps>0$ such that $B_{\eps}(0)\subset K$.
In other words, we have to show there is some $\eps>0$ such that
$\norm{\fc}\les\eps$ implies $\opt_{\fc}\les\yi$, independently
of $\fc$. From the mechanical point of view this means that every
force on the object may be scaled as to be supported by admissible
grasping forces, and there is a lower bound to these scaling factors.

To show the existence of $\eps$ as above, we first define the \textit{sensitivity
of the grasping} as
\begin{equation}
R:=\sup_{\fc\in\vs\du}\{R_{\fc}\}=\sup_{\fc\in\vs\du}\left\{ \frac{\opt_{\fc}}{\norm{\fc}}\right\} .\label{eq:Define_R}
\end{equation}
Thus, $R$ indicates the worst-case sensitivity of the grasping mechanism.
\begin{prop}
\label{prop:R}Let $\wh{\jac}:\vs\to\image\jac$ be the isomorphism
between $\vs$ and its image under $\jac$. Then,
\begin{equation}
R=\sup_{v\in\vs}\left\{ \frac{\norm v}{\norm{\jac(v)}}\right\} =\norm{\wh{\jac}^{-1}}.\label{eq:R}
\end{equation}
\end{prop}

\begin{proof}
Using Equation (\ref{eq:Optimal_Reactions}) of Proposition \ref{prop:Optimal},
we have
\begin{equation}
\begin{split}R & =\sup_{\fc\in\vs\du}\left\{ \frac{1}{\norm{\fc}}\sup_{v\in\vs}\left\{ \frac{\fc(v)}{\norm{\jac(v)}}\right\} \right\} ,\\
 & =\sup_{v\in\vs}\left\{ \frac{1}{\norm{\jac(v)}}\sup_{\fc\in\vs\du}\left\{ \frac{\fc(v)}{\norm{\fc}}\right\} \right\} ,\\
 & =\sup_{v\in\vs}\left\{ \frac{\norm v}{\norm{\jac(v)}}\right\} ,\\
 & =\sup_{\vf\in\image\jac}\left\{ \frac{\norm{\wh J^{-1}(\vf)}}{\norm{\vf}}\right\} ,\\
 & =\norm{\wh{\jac}^{-1}}.
\end{split}
\label{eq:RisNormJinv}
\end{equation}
\end{proof}
\begin{rem}
\label{rem:Find_Worst_Case}By Equation (\ref{eq:Define_R}), 
\begin{equation}
R=\sup\{\opt_{\fc}\mid\norm f=1,\,\fc\in\vs\du\}.
\end{equation}
Since $\vs\du$ is finite dimensional, the supremum above is attained
for some $\fc_{0}$ with $\norm{\fc_{0}}=1$---a worst case loading
distribution satisfying $R=\opt_{\fc_{0}}$. Since, on the other hand,
there is a generalized velocity $v_{0}$ such that 
\begin{equation}
R=\sup_{v\in\vs}\left\{ \frac{\norm v}{\norm{\jac(v)}}\right\} =\frac{\norm{v_{0}}}{\norm{\jac(v_{0})}},
\end{equation}
the worst case loading can be obtained form $v_{0}$ by the optimization
problem
\begin{equation}
\sup\left\{ \left.\frac{\fc(v_{0})}{\norm{\jac(v_{0})}}\;\right|\norm{\fc}=1\right\} .
\end{equation}
\end{rem}

\begin{cor}
\label{cor:one_over_R}One has
\begin{equation}
\inf_{\fc\in\vs\du}\left\{ \frac{\norm{\fc}}{\opt_{\fc}}\right\} =\frac{1}{R}>0.
\end{equation}
\end{cor}

\begin{proof}
It follows from the foregoing computation that
\begin{equation}
\begin{split}\inf_{\fc\in\vs\du}\left\{ \frac{\norm{\fc}}{\opt_{\fc}}\right\}  & =\frac{1}{\sup_{\fc\in\vs\du}\left\{ \frac{\opt_{\fc}}{\norm{\fc}}\right\} },\\
 & =\frac{1}{R},\\
 & =\frac{1}{\norm{\wh{\jac}^{-1}}}.
\end{split}
\end{equation}
Since $\wh{\jac}$ is a continuous isomorphism, $\norm{\wh{\jac}^{-1}}<\infty$
and does not vanish.
\end{proof}
\begin{cor}
\label{cor:eps_is_M_over_R}Let 
\begin{equation}
\eps:=\frac{\yi}{R}.
\end{equation}
Then, 
\begin{equation}
\inf\{\norm{\fc}\mid\fc\in\vs\du,\,\opt_{\fc}=\yi\}=\eps.\label{eq:eps_is_inf}
\end{equation}
\end{cor}

\begin{proof}
It follows from the foregoing corollary that 
\begin{equation}
\begin{split}\eps=\frac{\yi}{R} & =\inf_{\fc'\in\vs\du}\left\{ \left\Vert \frac{\yi\fc'}{\opt_{\fc'}}\right\Vert \right\} ,\\
 & =\inf\left\{ \left.\left\Vert \frac{\yi\fc'}{\opt_{\fc'}}\right\Vert \;\,\right|\frac{\yi\fc'}{\opt_{\fc'}}\in\vs\du\right\} .
\end{split}
\end{equation}
Observing that
\begin{equation}
\opt_{\yi\fc'/\opt_{\fc'}}=\yi,
\end{equation}
independently of $\fc'$, and writing $\fc=\yi\fc'/\opt_{\fc'}$,
one has
\begin{equation}
\eps=\inf\{\norm{\fc}\mid\opt_{\fc}=\yi\}.
\end{equation}
\end{proof}

\begin{cor}
A ball of radius $\eps=\yi/R$ is contained in $K$.
\end{cor}

It follows from the analysis presented above that indeed $K$ contains
an open neighborhood of the zero element in $\vs\du$ and the Minkowski
functional $p$ is well defined for $K$.

\subsection{\label{subsec:Sensitivity}The sensitivity of the grasping and the
load capacity ratio}

Once the validity of the Minkowski functional has been established,
we can turn to its applications in the study of bounded grasping reactions
and worst case loadings.

It is first noted that for a given force $\fc$, under scaling by
$1/p(\fc)$, the force $\fc/p(\fc)$ is located on the boundary of
$K$. Hence,

\begin{equation}
\opt_{\fc/p(\fc)}=\yi.
\end{equation}
However, as we have already seen that $\opt_{\yi\fc/\opt_{\fc}}=\yi$,
\begin{equation}
p(\fc)=\frac{\opt_{\fc}}{\yi}.
\end{equation}

Let 
\begin{equation}
\fs(\fc):=\frac{1}{p(\fc)}=\frac{\yi}{\opt_{\fc}}.\label{eq:Def_Safety_Factor}
\end{equation}
Then, if $\fc\ne0$ so that $p(\fc)\ne0$,
\begin{equation}
\begin{split}\fs(\fc) & =\frac{1}{p(\fc)},\\
 & =\frac{1}{\inf\{r\in\reals^{+}\mid\fc/r\in K\}},\\
 & =\sup\{1/r\mid r\in\reals^{+},\,\fc/r\in K\},\\
 & =\sup\{s\in\reals^{+}\mid s\fc\in K\}.
\end{split}
\label{eq:fac_safe_1}
\end{equation}
Thus, $\fs(\fc)$ is interpreted naturally as the \textit{factor of
safety,} the largest positive number by which we can multiply $\fc$
so that it can be supported by admissible grasping forces.

In various situations there are uncertainties as to the nature of
forces that will act on the object. In particular, one would like
to characterize the forces that the grasping mechanism will be able
to support. This is relevant especially in the case of bounded grasping.
For example, one would like to identify the loading conditions which
are most likely to cause inadmissible grasping forces so that grasping
fails. The following proposition gives a bound on the norm of a force
on the object that will ensure its admissibility, independently of
its distribution. In addition to the obvious dependence of this bound
on the value of $\yi$, it turns out that the bound depends only on
the geometry of the grasping mechanism as reflected by the mapping
$\jac$.
\begin{prop}
\label{prop:LoadCap}Any force $\fc$ acting on the object is admissible
if (we use the foregoing notation)
\begin{equation}
\norm{\fc}\les\frac{\yi}{R}=\frac{\yi}{\norm{\wh{\jac}^{-1}}}.
\end{equation}
The condition above is optimal in the sense that for $\gd>\yi/R$,
there is some force $\fc$ with $\norm f=\gd$ that cannot be supported
by the grasping mechanism. Hence, we will refer to $\yi/R$ as the
\emph{load capacity of the grasping}.
\end{prop}

\begin{proof}
To prove the claim, we should show that
\begin{equation}
\frac{\yi}{R}=\sup\{\gd\mid B_{\gd}(0)\subset K,\,\gd>0\}=\sup\{\gd\mid\opt_{\fc}\les\yi\text{ if }\norm{\fc}\les\gd\}
\end{equation}
(see illustration in Figure \ref{fig:load_cap}).
\begin{figure}
\begin{centering}
\includegraphics[scale=0.6]{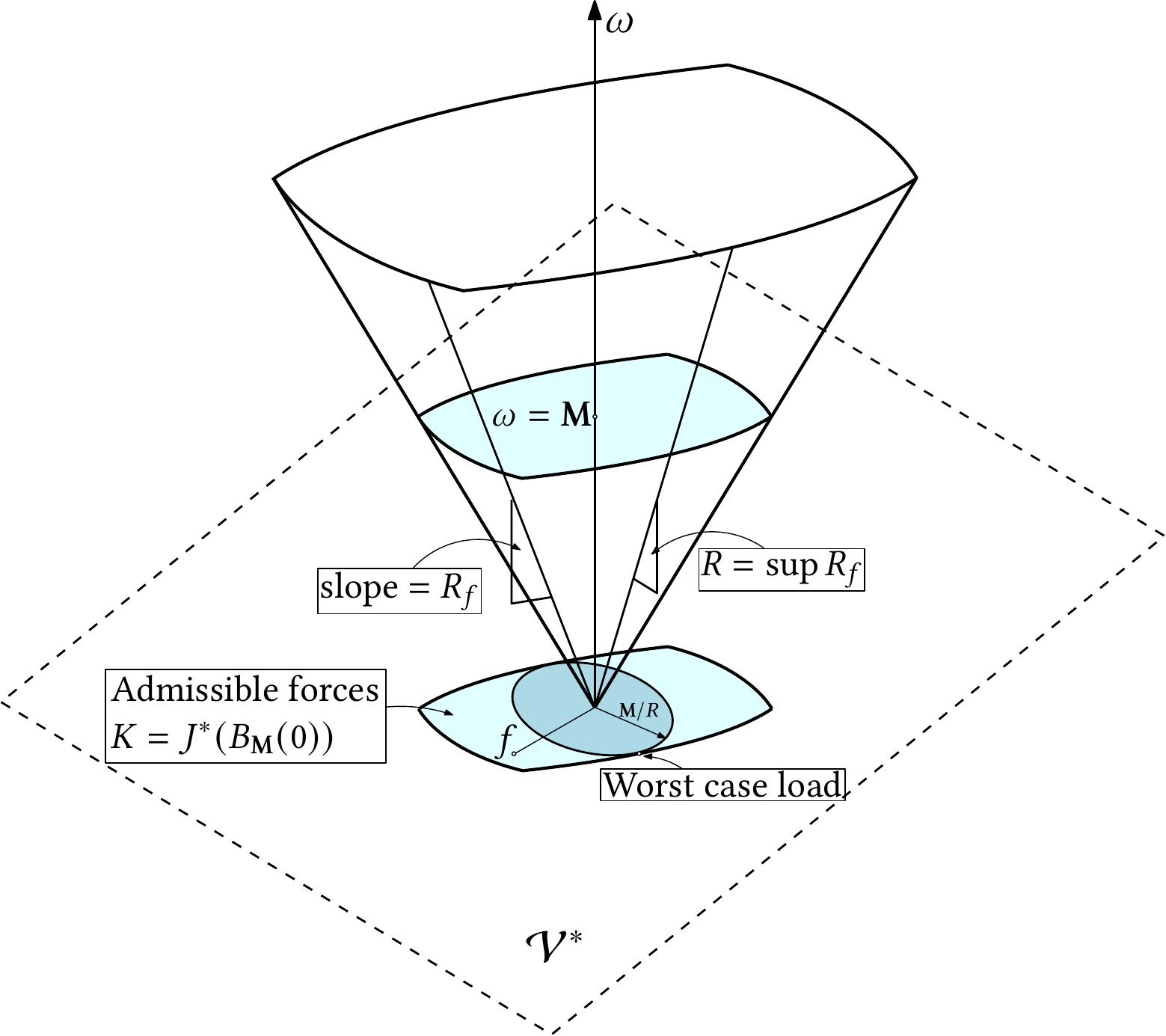}
\par\end{centering}
\caption{\label{fig:load_cap}Illustrating the various objects considered in
Proposition \ref{prop:LoadCap}.}
\end{figure}

Given $\gd>0$, consider the condition, 
\begin{equation}
\opt_{\fc}\les\yi\text{ if }\norm{\fc}\les\gd.\label{eq:Condition}
\end{equation}
Since $\opt_{\fc}=R_{\fc}\norm{\fc}$ and $R_{\fc}>0$, condition
(\ref{eq:Condition}) is equivalent to 
\begin{equation}
\begin{split}\yi & \ges\sup\{\opt_{\fc}\mid\norm{\fc}\les\gd\},\\
 & =\sup\{R_{\fc}\norm{\fc}\mid\norm{\fc}\les\gd\},\\
 & =\sup\{R_{\fc}\gd\mid\norm f=\gd\},\\
 & =\gd R,
\end{split}
\end{equation}
where we have used the fact that $R_{\fc}$ is independent of $\norm{\fc}$
as in Remark \ref{rem:Rf_norm-indep}, and the definition of $R$
in Equation (\ref{eq:Define_R}).

It follows that 
\begin{equation}
\begin{split}\sup\{\gd\mid B_{\gd}(0)\subset K,\,\gd>0\} & =\sup\{\gd\mid\gd R\les\yi\},\\
 & =\frac{\yi}{R}.
\end{split}
\end{equation}
\end{proof}

\subsection{Coulomb friction as bounded grasping}

As an example of bounded grasping, consider the case where a gripper
applies known normal forces to an object and we are concerned with
constraining the motion of the object by the resulting friction forces
in the plane perpendicular to the applied forces. That is, using sliding,
approach, and normal vectors as a basis for a coordinate system attached
to the gripper, as for example in \cite[p. 72]{Wolovich}, the normal
force acts along the sliding direction and the friction force acts
in the approach -- normal plane. Thus, under the assumption of Coulomb
friction, the friction grasping force is bounded. Some further remarks
follow.

\subsubsection{A general framework}

We start by setting up a general framework that is suitable for the
analysis of various examples where Coulomb friction plays an important
role. Assume that the vector space $\avs$ has a structure of a Cartesian
product $\avs=\avs_{1}\times\avs_{2}.$ We denote the projections
on the two subspaces by
\begin{equation}
\pi_{a}:\avs\tto\avs_{a},\qquad a=1,2.
\end{equation}
We will also use the notation 
\begin{equation}
\jac_{a}=\pi_{a}\comp\jac:\vs\tto\avs_{a},\qquad a=1,2.
\end{equation}

We recall that $\avs^{*}=\avs_{1}^{*}\times\avs_{2}^{*}$, and that
for $(\afc_{1},\afc_{2})\in\avs\du$, 
\begin{equation}
\jac^{*}(\afc_{1},\afc_{2})=\jac_{1}^{*}(\afc_{1})+\jac_{2}^{*}(\afc_{2}).\label{eq:dual_map_cart_prod}
\end{equation}

For a pair $(\afc_{1},\afc_{2})\in\avs\du$, $\afc_{1}\in\avs_{1}\du$
is interpreted as a known fixed grasping force vector, such as the
force in the sliding direction applied by a gripper, and so $\avs_{1}^{*}$
is the vector space of known fixed grasping forces. The friction grasping
forces are elements $\afc_{2}\in\avs_{2}\du$. Thus, as long as there
are no friction grasping forces, $\afc=(\afc_{1},0)$ and these grasping
forces equilibrate the \emph{initial} external force
\begin{equation}
\fc_{0}=\jac\du(\afc_{1},0)=\jac_{1}\du(\afc_{1}),
\end{equation}
which may be zero in many examples. It is noted that while $\jac$
is assumed to be injective, the component $\jac_{1}:\vs\to\avs_{1}$
is not assumed to be injective necessarily for this may mean that
the constraints preventing motion in $\avs_{1}$ are sufficient to
fix the object.

The assumption that we have Coulomb friction at the supports, implies
that we have a norm $\norm{\avf_{2}}_{2}$ on $\avs_{2}$ induced
by a norm $\norm{\afc_{2}}_{2}$ on $\avs_{2}^{*}$. The norm $\norm{\afc_{2}}_{2}$
depends in general one the values of the fixed reaction $\afc_{1}$.

Assume now that an additional external force $\fc$ is applied on
the mechanism so that the total external force is $\fc_{0}+\fc$.
Thus, the condition that the reactions $\afc=(\afc_{1},\afc_{2})$
equilibrate the total external force is 
\begin{equation}
\fc_{0}+\fc=\jac^{*}(\afc_{1},\afc_{2}).
\end{equation}
However, using (\ref{eq:dual_map_cart_prod})
\begin{equation}
\fc_{0}+\fc=\jac_{1}^{*}(\afc_{1})+\jac_{2}^{*}(\afc_{2}),
\end{equation}
and the assumption that the reactions $\afc_{1}$ are given and equilibrate
$\fc_{0}$ gives
\begin{equation}
\fc=\jac_{2}^{*}(\afc_{2}).\label{eq:Equil-Friction}
\end{equation}
Thus, all the observations made in Sections \ref{subsec:Optimum}
and \ref{subsec:Sensitivity} apply to the situation described above
in case $\jac_{2}$ is injective. Note that the assumption that $\jac_{2}$
is injective implies that any force $\fc$ may be supported by friction
alone (under the given normal forces $\afc_{1}$). This is summarized
by the following.
\begin{prop}
\label{prop:Result-Friction} Assume that the mapping $\jac_{2}:\vs\to\avs_{2}$
is injective. There is an optimal reaction $\afc_{2\textrm{H}}\in\avs_{2}^{*}$
(not unique)
\begin{equation}
\opt_{\fc}:=\norm{\afc_{2\mathrm{H}}}=\inf_{f=\jac_{2}^{*}(g_{2})}\norm{g_{2}}
\end{equation}
 for which (\ref{eq:Equil-Friction}) holds. It satisfies
\begin{equation}
\opt_{\fc}=\norm{\afc_{2\mathrm{H}}}=\sup_{v\in\vs}\frac{\fc(v)}{\norm{\jac_{2}(v)}}.\label{eq:Expression_Optimum-1}
\end{equation}
In addition, let
\begin{equation}
R:=\sup_{\fc\in\vs^{*}}\frac{\norm{\opt_{\fc}}}{\norm{\fc}}.\label{eq:Sensitivity_Ratio-1}
\end{equation}
Then, 
\begin{equation}
R=\sup_{v\in\vs}\frac{\norm v}{\norm{\jac_{2}(v)}}.
\end{equation}
\end{prop}

\subsubsection{A constitutive model for friction}

We assume that at a generic constraint point, say $\ga$, where the
reaction is some given $\afc_{1\ga}\ges0$, the Euclidean norm of
the maximal friction grasping force is given by $\norm{\afc_{2\ga}}$
which satisfies
\begin{equation}
\norm{\afc_{2\ga}}\les\mu_{\ga}\afc_{1\ga},
\end{equation}
for some known coefficient $\mu_{\ga}$ of static friction. As $\mu_{\ga}$
and $\afc_{\ga}$ are assumed to be given, we may write this condition
as 
\begin{equation}
\left\Vert \frac{\afc_{2\ga}}{\mu_{\ga}\afc_{1\ga}}\right\Vert \les1.
\end{equation}
Grasping with friction forces is admissible if the last condition
holds for every value of $\ga$, or equivalently, if
\begin{equation}
\sup_{\ga}\left\Vert \frac{\afc_{2\ga}}{\mu_{\ga}\afc_{1\ga}}\right\Vert \les1.
\end{equation}
The condition suggests that for the norm of the friction forces we
use the weighted norm
\begin{equation}
\norm{\afc_{2}}=\sup_{\ga}\{\nu_{\ga}\norm{\afc_{\ga}}\},\qquad\nu_{\ga}=\frac{1}{\mu_{\ga}\afc_{1\ga}}.\label{eq:Norm_Friction}
\end{equation}

In view of Equation (\ref{eq:Equil-Friction}), the notions introduced
in Section \ref{subsec:Sensitivity} apply to the case of friction
with $\yi_{2}=1$, where $\yi_{2}$ is the grasping bound for admissible
$\afc_{2}\in\avs_{2}\du$ in analogy with the grasping bound $\yi$
introduced in Section \ref{sec:Load_Capacity}.
\begin{rem}
It follows from the setting described above, that in case 
\begin{equation}
\opt_{\fc}:=\inf_{\jac_{2}^{*}(\afc_{2})=f}\norm{\afc_{2}}>1,
\end{equation}
then, $f$ cannot be supported by friction grasping.
\end{rem}

For the factor of safety, as in (\ref{eq:Def_Safety_Factor}) and
(\ref{eq:fac_safe_1}), we therefore have to the case of friction
\begin{equation}
\fs(\fc)=\sup\{s>0\mid\opt_{s\fc}=s\opt_{\fc}\les1\},
\end{equation}
and
\begin{equation}
\fs(\fc)=\frac{1}{\opt_{\fc}}.\label{eq:Mult-Optim-Friction}
\end{equation}

The main assumption regarding static friction is:

\bigskip{}

\textbf{\textit{Constitutive assumption for Coulomb friction:}}\textit{\emph{
}}\emph{A force $\fc$ is supported by (not necessarily unique) friction
forces if 
\begin{equation}
\opt_{\fc}\les1.
\end{equation}
}

Thus, in terms of safety factor, a force $f$ is supported by friction
reactions if and only if 
\begin{equation}
\fs(\fc)\ges1.
\end{equation}

\begin{rem}
Consider an external force $\fc$ such that $\opt_{\fc}<1$, strictly.
The constitutive assumption above implies that there are some friction
reactions $\afc_{2}$, with $\norm{\afc_{2}}\les1$ and $\jac^{*}(\afc_{2})=\fc$.
Thus, in general, $\norm{\afc_{2}}\ges\opt_{\fc}$ and $\afc_{2}$
is not necessarily optimal. For a force $\fc$ with $\opt_{\fc}=1$,
any friction forces $\afc_{2}$ with $\jac^{*}(\afc_{2})=\fc$, must
satisfy $\norm{\afc_{2}}=1=\opt_{\fc}$ and hence, $\afc_{2}$ is
optimal. We conclude that the frictions forces for a force $\fc$
with $\opt_{\fc}=1$ are optimal.
\end{rem}

Next, we consider the implications of the constitutive assumption
for friction forces for the sensitivity ratio $R$. It follows from
Equation (\ref{eq:Define_R}) that for each external loading $\fc$,
\begin{equation}
\frac{\norm{\afc_{2\mathrm{H}}}}{\norm{\fc}}=\frac{\opt_{\fc}}{\norm{\fc}}\les R,\qquad\text{or,}\qquad\norm{\fc}R\ges\norm{\afc_{2\mathrm{H}}}=\opt_{\fc}.
\end{equation}

\begin{cor}
\label{cor:Res_Fric_Gasp}For any external force $\fc$ such that
\begin{equation}
\norm f\les\frac{1}{R},\qquad\text{we have,}\qquad\opt_{\fc}\les1.\label{eq:Sens_Appl_Fric}
\end{equation}
\emph{It then follows from the constitutive assumption for friction
that any external force $\fc$ such that $\norm{\fc}\les1/R$ may
be supported by friction grasp, independently of its distribution,
direction, etc.}
\end{cor}

\subsection{Example: The sensitivity for a 2-dimensional object held at two points}

We return to system considered in Example \ref{subsec:Example:wafer}
and Figure \ref{fig:wafer1}, and we want to compute the sensitivity
$R$ of the grasping as given by Equation (\ref{eq:R}), Proposition
\ref{prop:R}. This will enable us to evaluate the capacity of the
system to support arbitrary external forces.

The sensitivity $R$ is defined relative to a norm on the space, $\vs\du=\{(f_{x},f_{y},M)\}$,
of all external forces on the object which is the dual norm of some
norm on the space of generalized velocities $\vs=\{(v_{x},v_{y},\go)\}$.
Thus, the norm chosen in $\vs\du$ is
\begin{equation}
\norm{\fc}=\norm{(f_{x},f_{y},M)}=\sqrt{\fc_{x}^{2}+\fc_{y}^{2}+(M/r)^{2}}.
\end{equation}
The corresponding primal norm on $\vs$ is
\begin{equation}
\norm{\v}=\norm{(v_{x},v_{y},\go)}=\sqrt{v_{x}^{2}+v_{y}^{2}+(r\omega)^{2}}.
\end{equation}
Thus,
\begin{equation}
R=\sup_{v\in\vs}\left\{ \frac{\norm{\v}}{\norm{\jac(v)}}\right\} =\sup_{\v\in\vs}\frac{\sqrt{v_{x}^{2}+v_{y}^{2}+(r\omega)^{2}}}{\sqrt{(v_{x}+r\omega)^{2}+v_{y}^{2}}+\sqrt{(v_{x}-r\omega)^{2}+v_{y}^{2}}}.
\end{equation}
Using the notation
\begin{equation}
\alpha=v_{x}^{2}+v_{y}^{2}+(r\omega)^{2},\quad\beta=2v_{x}r\omega,
\end{equation}
one obtains,
\begin{equation}
\begin{split}R & =\sup_{\alpha,\beta}\frac{\sqrt{\alpha}}{\sqrt{\alpha-\beta}+\sqrt{\alpha+\beta}},\\
 & =\sup_{\alpha,\beta}\sqrt{\frac{\alpha}{\alpha-\beta+\alpha+\beta+2\sqrt{\alpha+\beta}\sqrt{\alpha-\beta}}},\\
 & =\sup_{\alpha,\beta}\sqrt{\frac{\alpha}{2\alpha+2\sqrt{\alpha+\beta}\sqrt{\alpha-\beta}}},\\
 & =\sqrt{\frac{\alpha}{2\alpha}},
\end{split}
\end{equation}
and it follows that 
\begin{equation}
R=\frac{1}{\sqrt{2}}.
\end{equation}

We conclude from Proposition \ref{prop:LoadCap}, that under a force
controlled grasping bounded by $\yi$, the system can support an arbitrary
combination of an external force and a moment, $\fc=(f_{x},f_{y},M)$,
as long as
\begin{equation}
\sqrt{\fc_{x}^{2}+\fc_{y}^{2}+(M/r)^{2}}\les\frac{\yi}{R}=\sqrt{2}\yi.
\end{equation}

\subsection{\label{subsec:Example-R-for_Beam}Example: Grasping sensitivity for
a 2-D object supported at three points}

We return to the system considered in Section \ref{subsec:Example_Beam},
Figure \ref{fig:3points}, where the weight of a 2-dimensional object
is supported at three points, and we calculate the sensitivity of
the grasping. As a norm for the external load, we choose
\begin{equation}
\norm{\fc}=\sqrt{\fc_{y}^{2}+(M/l)^{2}}=m\g\sqrt{1+(a/l)^{2}}.
\end{equation}
The corresponding primal norm is 
\begin{equation}
\norm v=\sqrt{v_{y}^{2}+(l\omega)^{2}}.
\end{equation}
Thus, 
\begin{equation}
R=\sup_{v\in\vs}\left\{ \frac{\norm{\v}}{\norm{\jac(v)}}\right\} =\sup_{\v\in\vs}\frac{\sqrt{v_{y}^{2}+(l\omega)^{2}}}{\abs{v_{y}-l\omega}+\abs{v_{y}}+\abs{v_{y}+l\omega}}.
\end{equation}
Analyzing the function in the argument of the supremum for various
relative values of $v_{y}$ and $l\go$, the value $R=1/2$ is obtained.

\subsection{\label{subsec:Example-mechanism}Example: The sensitivity of grasping
for a mechanism}

This example demonstrates the application of the foregoing analysis
for the case where the object is a mechanism rather than a rigid body.
Consider the system illustrated in Figure \ref{fig:mechanism} where
the object is a two-dimensional, two-link mechanism. It is assumed
that the gripping force (in the $z$-direction) exerted by the left
gripper is given. This implies that the Coulomb friction force in
the $x-y$ plane, which the gripper applies to the mechanism on the
left, is bounded. We denote this maximal friction force by $\afc_{\yi}$.
The arm on the right end of the mechanism is assumed to apply a fixed
force in the $-x$-direction. Thus, considering only the mechanics
of the mechanism in the $x-y$ plane, the vertical friction force
on the right hand side is bounded as well. It is assumed that the
vertical force on the right is bounded by $g_{\yi}$, also. In addition,
it is assumed that that there is friction in the joint located at
the origin. The friction at the joint enables friction moment $L_{\yi}$
that the links can apply to one another, and it is assumed that it
is bounded by $\abs{L_{\yi}}\les10\afc_{\yi}l$.

In the notation for this example, we do not use the decomposition
$\afc=(\afc_{1},\afc_{2})$ as we simply ignore the applied normal
forces which cancel one another. These normal forces just determine
the maximal friction forces so we will simply write $\afc$ for $\afc_{2}$.

\begin{figure}
\begin{centering}
\includegraphics[scale=0.6]{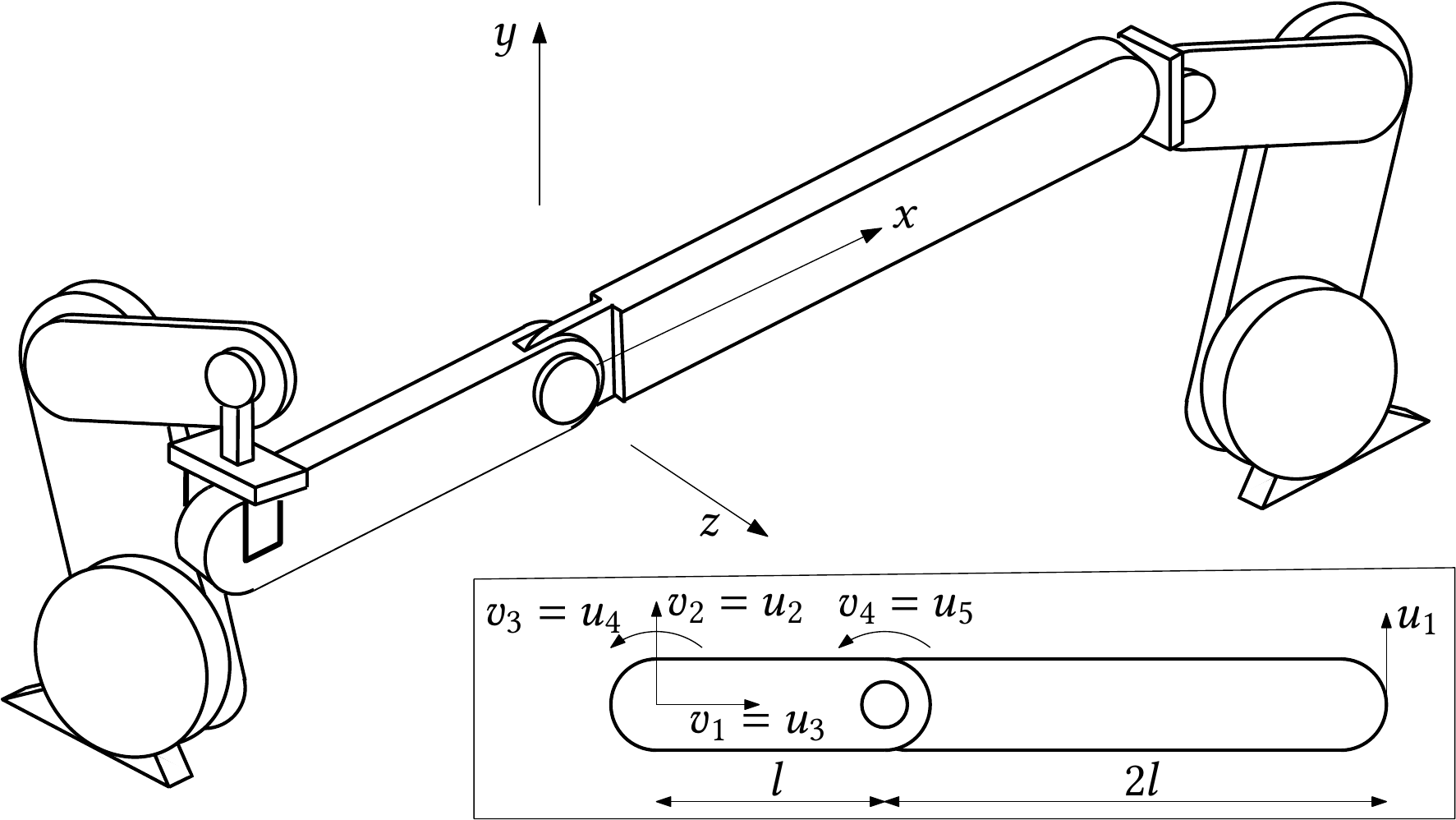}
\par\end{centering}
\caption{\label{fig:mechanism}Illustrating the system for Example \ref{subsec:Example-mechanism}.}
\end{figure}
Letting $\csa$ denote the configuration space of the two-link mechanism---the
object in this example---and considering the configuration $\conf$
for which the mechanism is flat (see Figure \ref{fig:mechanism}),
we have
\begin{equation}
\vs=T_{\conf}\csa\isom\{(v_{1}=:v_{x},v_{2}=:v_{y},v_{3}=:\go_{1},v_{4}=:\go_{2})\},
\end{equation}
where $v_{1},v_{2}$ are the horizontal and vertical component of
the velocity of the left end, respectively, $v_{3}$ is the angular
velocity of the left link, and $v_{4}$ is the angular velocity of
the right link relative that of the left link. Elements of the configuration
space for the collection of the supports, $\csb$, contain in addition
the location of the right end of the right-hand link, and so
\begin{equation}
\avs=T_{\fk(\conf)}\csb\isom\{(\avf_{1},\avf_{2}:=v_{y},\avf_{3}:=v_{x},\avf_{4}:=\go_{1},\avf_{5}:=\go_{2})\},
\end{equation}
where $\avf_{1}$ is the vertical velocity of the right end (see Figure
\ref{fig:mechanism}).

Correspondingly, the spaces of generalized external forces and generalized
forces at the supports may be written as 
\begin{equation}
\vs\du=\{(\fc_{x},\fc_{y},M_{1},M_{2})\},\qquad\avs\du=\{(\afc_{1},\dots,\afc_{5})\}
\end{equation}
where in particular, $\afc_{1}$ is the vertical friction force acting
on the right end.

A simple kinematic analysis of the mechanism gives
\begin{equation}
\jac=\left(\begin{array}{cccc}
0 & 1 & 3l & 2l\\
0 & 1 & 0 & 0\\
1 & 0 & 0 & 0\\
0 & 0 & 1 & 0\\
0 & 0 & 0 & 1
\end{array}\right),\qquad\image\jac=\left\{ \left(\begin{array}{c}
\avf_{1}=v_{y}+3l\omega_{1}+2l\omega_{2}\\
\avf_{2}=v_{y}\\
\avf_{3}=v_{x}\\
\avf_{4}=\go_{1}\\
\avf_{5}=\go_{2}
\end{array}\right)\right\} .
\end{equation}

As a norm on the space of grasping forces we use
\begin{equation}
\norm{\afc}=\frac{1}{\afc_{\yi}}\max\left\{ \afc_{1},\afc_{2},\afc_{3},\frac{\afc_{4}}{l},\frac{\afc_{5}}{10l}\right\} .
\end{equation}
(It is noted that it would be more physical to take $\sqrt{g_{2}^{2}+\afc_{3}^{2}}$,
rather than the maximum between these two components. However, our
choice is motivated by simplicity and by the fact that the Euclidean
norm is bounded by the maximum.) For the norm of the external loading
we chose
\begin{equation}
\norm{\fc}=\sqrt{\fc_{x}^{2}+\fc_{y}^{2}+\left(\frac{M_{1}}{l}\right)^{2}+\left(\frac{M_{2}}{l}\right)^{2}}.
\end{equation}
The corresponding primal norms are
\begin{equation}
\norm{\avf}=\afc_{\yi}(\abs{\avf_{1}}+\abs{\avf_{2}}+\abs{\avf_{3}}+\abs{\avf_{4}l}+\abs{10\avf_{5}l})\quad\text{and}\quad\norm v=\sqrt{v_{x}^{2}+v_{y}^{2}+(l\go_{1})^{2}+(l\go_{2})^{2}}.
\end{equation}

We can now write down the expression for the sensitivity as
\begin{equation}
\begin{split}R= & \sup_{v\neq0}\frac{\norm v}{\norm{\jac(v)}}\\
 & \sup_{v\neq0}\frac{\sqrt{v_{x}^{2}+v_{y}^{2}+(l\omega_{1})^{2}+(l\omega_{2})^{2}}}{\afc_{\yi}\abs{v_{y}+3l\omega_{1}+2l\omega_{2}}+\afc_{\yi}\abs{v_{y}}+\afc_{\yi}\abs{v_{x}}+\afc_{\yi}l\abs{\omega_{1}}+10\afc_{\yi}l\abs{\omega_{2}}}.
\end{split}
\end{equation}
We were not able to compute the supremum analytically. A numerical
computation using the commercial platform modeFronier resulted the
value $R=1.00$. It was also observed that the result is independent
of $l$ and the same result was obtained for links having equal lengths.

\subsection{\label{subsec:Example_Glass}Example: Grasping sensitivity for holding
a wineglass}

As a 3-dimensional example, we consider next the grasping sensitivity
of a wineglass held at four points as shown in Figure \ref{fig:glass}.
The four fingers hold the glass on the same horizontal circle of maximal
radius, $r$. The two pairs of diametrically opposed fingers are along
the $x$-axis and $y$-axis, respectively. 
\begin{figure}
\begin{centering}
\includegraphics[scale=0.75]{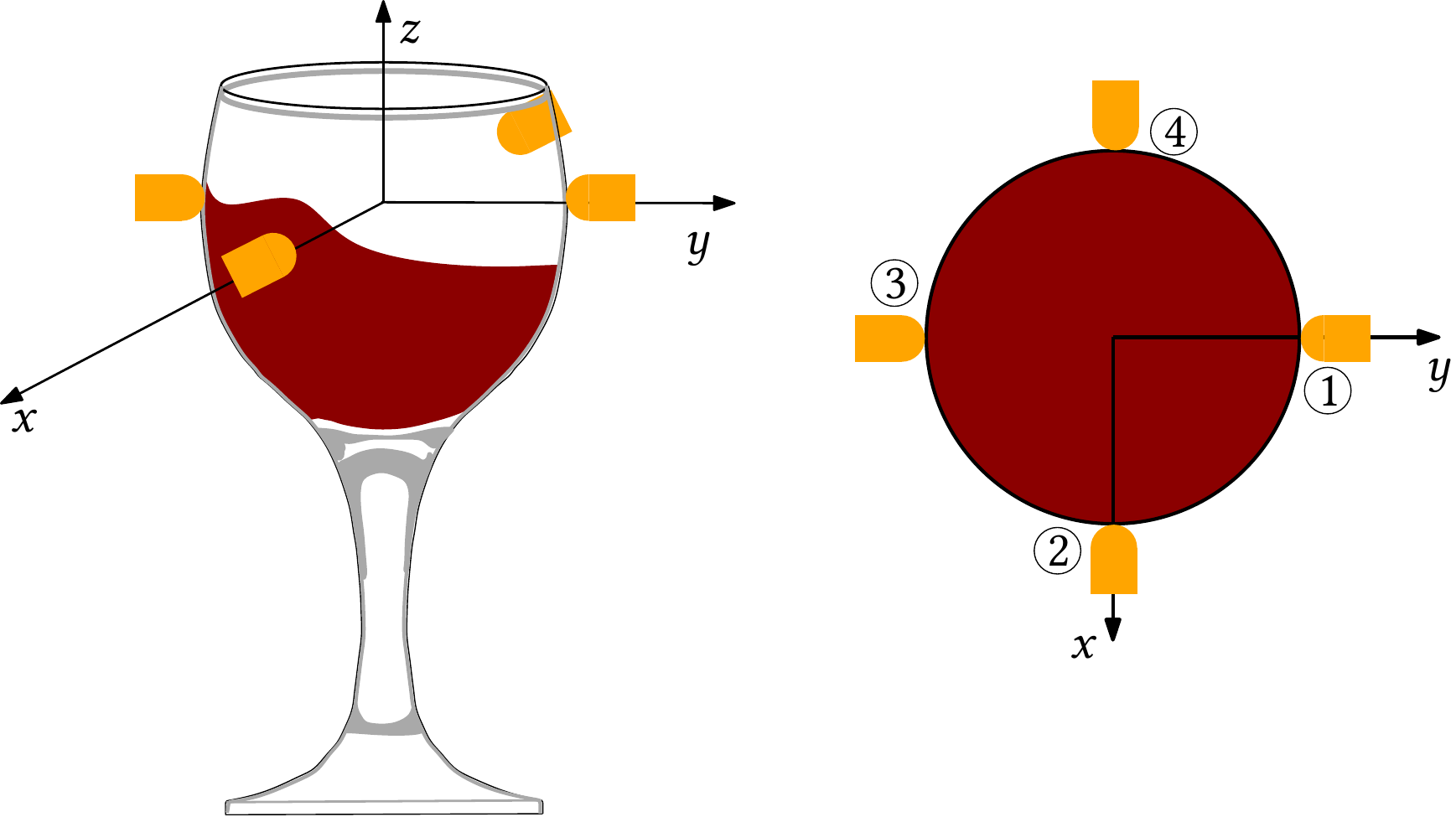}
\par\end{centering}
\caption{\label{fig:glass}Illustrating a wineglass held at four points}
\end{figure}

It is assumed that all four fingers apply equal normal grasping forces
on the glass which determine the maximal Coulomb friction, $\afc_{\yi}$,
at these points. The generalized velocity of the glass are described
by vectors $v\in\vs$ of the form $v=(v_{x},v_{y},v_{z},\go_{x},\go_{y},\go_{z})$
consisting of the components of the linear velocity of the center
of the glass (located at the origin) and the angular velocity.

Since we will be interested in the components of friction forces at
the four contact points, there are two relevant components at each
point. Thus, adopting the enumeration of the fingers as shown on the
right of Figure \ref{fig:glass}, the grasping generalized force $\afc\in\avs\du$
is of the form $\afc=(\afc_{1x},\afc_{1z},\afc_{2y},\afc_{2z},\afc_{3x},\afc_{3z},\afc_{4y},\afc_{4z})$.
It follows that the constraints generalized velocities $\avf\in\avs$
are represented in the form $\avf=(\avf_{1x},\avf_{1z},\avf_{2y},\avf_{2z},\avf_{3x},\avf_{3z},\avf_{4y},\avf_{4z})$.
Kinematics implies that 
\begin{equation}
\jac=\left(\begin{array}{cccccc}
1 & 0 & 0 & 0 & 0 & -r\\
0 & 0 & 1 & r & 0 & 0\\
0 & 1 & 0 & 0 & 0 & r\\
0 & 0 & 1 & 0 & -r & 0\\
1 & 0 & 0 & 0 & 0 & r\\
0 & 0 & 1 & -r & 0 & 0\\
0 & 1 & 0 & 0 & 0 & -r\\
0 & 0 & 1 & 0 & r & 0
\end{array}\right),
\end{equation}

As in the preceding example, we will not use the decomposition $\afc=(\afc_{1},\afc_{2})$
in the notation and we will simply write $\afc$ for $\afc_{2}$.
Next, we have to specify the norms chosen for $\avs\du$, $\vs\du$
and the corresponding primal vector spaces. Thus, denoting the maximal
admissible friction component by $\afc_{\yi}$, we use the norm
\begin{equation}
\norm{\afc}=\frac{1}{\afc_{\yi}}\max\{\abs{\afc_{1x}},\abs{\afc_{1z}},\abs{\afc_{2y}},\abs{\afc_{2z}},\abs{\afc_{3x}},\abs{\afc_{3z}},\abs{\afc_{4y}},\abs{\afc_{4z}}\},
\end{equation}
for elements of $\avs\du$ and the primal norm on $\avs$ is
\begin{equation}
\norm{\avf}=\afc_{\yi}(\abs{u_{1x}}+\abs{u_{1z}}+\abs{u_{2y}}+\abs{u_{2z}}+\abs{u_{3x}}+\abs{u_{3z}}+\abs{u_{4y}}+\abs{u_{4z}}).
\end{equation}
(The max-norm is used instead of the Euclidean for the friction at
each point in order to simplify the computations and because the Euclidean
norm is bounded by the max-norm.) For the generalized external forces,
represented in the form $\fc=(\fc_{x},\fc_{y},\fc_{z},M_{x},M_{y},M_{z})$,
we use the norm
\begin{equation}
\norm{\fc}=\sqrt{f_{x}^{2}+f_{y}^{2}+\fc_{z}^{2}+(M_{x}/r)^{2}+(M_{y}/r)^{2}+(M_{z}/r)^{2}},
\end{equation}
which gives the primal norm
\begin{equation}
\norm v=\sqrt{v_{x}^{2}+v_{y}^{2}+v_{z}^{2}+(r\omega_{x})^{2}+(r\omega_{y})^{2}+(r\omega_{z})^{2}}.
\end{equation}

It is concluded that the sensitivity of the grasping system is given
by
\begin{equation}
\begin{split}R & =\sup_{v\neq0}\frac{\norm v}{\norm{\jac(v)}}\\
 & =\frac{1}{\afc_{\yi}}\sup_{v\neq0}\frac{\sqrt{v_{x}^{2}+v_{y}^{2}+v_{z}^{2}+(r\omega_{x})^{2}+(r\omega_{y})^{2}+(r\omega_{z})^{2}}}{D},
\end{split}
\end{equation}
where
\begin{multline}
D=\abs{v_{x}-\omega_{z}r}+\abs{v_{z}+\omega_{x}r}+\abs{v_{y}+\omega_{z}r}+\abs{v_{z}-\omega_{y}r}\\
+\abs{v_{x}+\omega_{z}r}+\abs{v_{z}-\omega_{x}r}+\abs{v_{y}-\omega_{z}r}+\abs{v_{z}+\omega_{y}r}.
\end{multline}
Numerical computations using modeFrontier  give the approximate
value $R=0.500/\afc_{\yi}$. We obtained the same result for a number
of values for $r$. Thus finally, the grasping can support a combination
of external forces and external moments as long as their norms satisfy
$\norm{\fc}\les2.00\afc_{\yi}$.

\bigskip{}

\noindent \textbf{\textit{Acknowledgments.}} This work was partially
supported by the H.~Greenhill Chair for Theoretical and Applied Mechanics
and by the Pearlstone Center for Aeronautical Engineering Studies
at Ben-Gurion University.

\appendix

\section{\label{sec:H-B}The Helly Construction for the Extension of Linear
Functionals}

We review here the construction of the Hahn-Banach theorem (in the
finite dimensional case) for the extension of covectors. This construction
will be applied in some of the examples.

Consider a normed finite dimensional vector space $\avs$, $\dimension\avs=m$
and a proper vector subspace $\svs\subset\avs$, $\dimension\svs=n<m$.
Let $\afc_{0}\in\svs\du$ be a linear functional with 
\begin{equation}
\norm{\afc_{0}}=\sup_{\substack{\vf\neq0\\
\vf\in\svs
}
}\frac{\abs{\afc_{0}(\vf)}}{\norm{\vf}}
\end{equation}
be given. An extension $\afc\in\avs\du$ of $\afc_{0}$ is a covector
that satisfies the condition $\afc(\vf)=\afc_{0}(\vf)$ for all $\vf\in\svs$.
Some extension $\afc$ of $\afc_{0}$ that preserves the norm, that
is
\begin{equation}
\norm{\afc}:=\sup_{\substack{\avf\neq0\\
\avf\in\avs
}
}\frac{\abs{\afc(\avf)}}{\norm{\avf}}=\norm{\afc_{0}},
\end{equation}
is constructed as follows.

Since we remain in the finite dimensional setting, it is sufficient
to provide the construction for the case where $\svs$ is a subspace
of co-dimension $m-n=1$. For the general case, one can repeat the
process $(m-n)$-times. Thus, let $\avf_{1}\in\avs\setminus\svs$
be chosen and so every vector $u\in\avs$ may be expressed as 
\begin{equation}
\avf=\vf+t\avf_{1},
\end{equation}
for unique $\vf\in\svs$ and $t\in\reals$. It follows from linearity,
that every extension $\afc\in\avs\du$ satisfies
\begin{equation}
\afc(\avf)=\afc(\vf+t\avf_{1})=\afc_{0}(\vf)+t\afc(\avf_{1}).
\end{equation}
This condition implies that every extension $\afc$ is uniquely determined
by a single number $\afc(\avf_{1})$. One has to show, therefore,
that $\afc(\avf_{1})$ may be chosen so that $\norm{\afc}=\norm{\afc_{0}}$.

It is observed first, that 
\begin{equation}
\begin{split}\norm{\afc} & =\sup_{\substack{\vf\in\svs\\
t\in\reals
}
}\frac{\abs{\afc_{0}(\vf)+t\afc(\avf_{1})}}{\norm{\vf+t\avf_{1}}},\qquad t\ne0,\\
 & =\sup_{\substack{\vf\in\svs\\
t\in\reals
}
}\frac{\abs t\abs{\afc_{0}(\frac{\vf}{t})+\afc(\avf_{1})}}{\abs t\norm{\frac{\vf}{t}+\avf_{1}}},\qquad t\ne0,\\
 & =\sup_{\vf\in\svs}\frac{\abs{\afc_{0}(\vf)+\afc(\avf_{1})}}{\norm{\vf+\avf_{1}}}.
\end{split}
\end{equation}
Thus, the condition that $\norm{\afc}=\norm{\afc_{0}}$ will be satisfied
if for all $\vf\in\svs$,
\begin{equation}
\norm{\afc_{0}}\norm{\vf+\avf_{1}}\ges\abs{\afc_{0}(\vf)+\afc(\avf_{1})}.
\end{equation}
This gives the following conditions on $\afc(\avf_{1})$:
\begin{equation}
\norm{\afc_{0}}\norm{\vf+\avf_{1}}-\afc_{0}(\vf)\ges\afc(\avf_{1})\fall\vf\in\svs,
\end{equation}
and 
\begin{equation}
-\norm{\afc_{0}}\norm{\vf+\avf_{1}}-\afc_{0}(\vf)\les\afc(\avf_{1})\fall\vf\in\svs,
\end{equation}
which replacing $\vf$ with $-\vf$ is equivalent to 
\begin{equation}
-\norm{\afc_{0}}\norm{-\vf+\avf_{1}}+\afc_{0}(\vf)\les\afc(\avf_{1})\fall\vf\in\svs.
\end{equation}
Together, the conditions are
\begin{equation}
\afc_{0}(\vf)-\norm{\afc_{0}}\norm{\avf_{1}-\vf}\les\afc(\avf_{1})\les\norm{\afc_{0}}\norm{\avf_{1}+\vf}-\afc_{0}(\vf).\label{eq:sup-inf H-B}
\end{equation}
To show that the two regions are indeed separated, so that such a
number $\afc(\avf_{1})$ exists, it is noted that for any $\vf,\vf'\in\svs$,
\begin{equation}
\afc_{0}(\vf)+\afc_{0}(\vf')=\afc_{0}(\vf+\vf')\les\norm{\afc_{0}}\norm{\vf+\vf'}\les\norm{\afc_{0}}\norm{\vf+\avf}+\norm{\afc_{0}}\norm{\vf'-\avf},
\end{equation}
where on the right, $\avf$ is an arbitrary element of $\avs$. It
follows that
\begin{equation}
\afc_{0}(\vf')-\norm{\afc_{0}}\norm{\vf'-\avf}\les\norm{\afc_{0}}\norm{\vf+\avf}-\afc_{0}(\vf)\fall\vf,\vf'\in\svs,\,\avf\in\avs.
\end{equation}


\begin{thebibliography}{Wol87}

\bibitem[Bol88]{Bologni1988}
L.~Bologni.
\newblock Robotic grasping: How to determine contact position.
\newblock In {\em 2nd IFAC Symposium on Robot Control}, pages 395--400, 1988.

\bibitem[BR19]{Burdick2019}
J.~Burdick and E.~Rimon.
\newblock {\em The Mechanics of Robot Grasping}.
\newblock Cambridge University Press, 2019.

\bibitem[FS09]{Falach2009}
L.~Falach and R.~Segev.
\newblock Load capacity ratios for structures.
\newblock {\em Computer Methods in Applied Mechmechanics and Engneering},
  199:77--93, 2009.

\bibitem[FS12]{Fasoulas2012}
J.~Fasoulas and M.~Sfakiotakis.
\newblock Modeling and control for object manipulation by a two {D.O.F.}
  robotic hand with soft fingertips.
\newblock In {\em 10th IFAC Symposium on Robot Control, International
  Federation of Automatic Control}, pages 259--264, 2012.

\bibitem[Hel12]{Helly1912}
E.~Helly.
\newblock {\"Uber lineare Funktionaloperationen}.
\newblock {\em {Wien. Ber.} {Akademie der Wissenschaften in Wien,
  Mathematisch-Naturwissenschaftliche Klasse, Sitzungsberichte, Abteilung
  IIa}}, 121:265--297, 1912.

\bibitem[Seg07]{LoadCap07}
R.~Segev.
\newblock Load capacity of bodies.
\newblock {\em International Journal of Non-Linear Mechanics}, 42 \textrm{(A
  special volume in memory of R.~Rivlin)}:250--257, 2007.

\bibitem[Wol87]{Wolovich}
W.A. Wolovich.
\newblock {\em Robotics: Basic Analysis and Design}.
\newblock Holt, Rinehart and Winston, 1987.

\bibitem[YO10]{Ohka2010}
H.~Yussof and M.~Ohka.
\newblock Analysis of tactile slippage control algorithm for robotic hand
  performing grasp-move-twist motions.
\newblock {\em International Journal on Smart Sensing and Intelligent Systems},
  3:359--375, 2010.

\end{thebibliography}
\end{document}